\documentclass{article} 
\usepackage{iclr_style_files/iclr2027_conference}
\usepackage{times}

\usepackage{amsmath,amsfonts,bm}

\newcommand{\figleft}{{\em (Left)}}

\newcommand{\figright}{{\em (Right)}}

\def\1{\bm{1}}

\DeclareMathAlphabet{\mathsfit}{\encodingdefault}{\sfdefault}{m}{sl}
\SetMathAlphabet{\mathsfit}{bold}{\encodingdefault}{\sfdefault}{bx}{n}

\newcommand{\softmax}{\mathrm{softmax}}

\DeclareMathOperator*{\argmin}{arg\,min}

\usepackage{amsmath}
\usepackage{amssymb}
\usepackage{mathtools}
\usepackage{amsthm}

\theoremstyle{plain}
\newtheorem{theorem}{Theorem}[section]
\newtheorem{proposition}[theorem]{Proposition}
\newtheorem{lemma}[theorem]{Lemma}
\newtheorem{corollary}[theorem]{Corollary}
\theoremstyle{definition}
\newtheorem{definition}[theorem]{Definition}
\newtheorem{assumption}[theorem]{Assumption}
\theoremstyle{remark}
\newtheorem{remark}[theorem]{Remark}

\usepackage[utf8]{inputenc} 
\usepackage[T1]{fontenc}    
\usepackage{hyperref}       
\usepackage{url}            
\usepackage[capitalize,noabbrev]{cleveref}
\usepackage{soul}
\usepackage{adjustbox}
\usepackage[svgnames,dvipsnames]{xcolor}
\usepackage{subcaption}
\usepackage{algorithm}
\usepackage{algorithmic}
\usepackage{booktabs}       
\usepackage{amsfonts}       
\usepackage{nicefrac}       
\usepackage{microtype}      
\usepackage{float}

\usepackage{mathdefs}

\newcommand{\alglinelabel}[1]{%
  \addtocounter{ALC@line}{-1}
  \refstepcounter{ALC@line}
  \label{algline:#1}
}

\renewcommand{\cite}{\citep}

\newcommand{\change}[1]{{\color{black}#1}}

\title{PMF-CL: Pareto-Minimal-Forgetting Continual Learner for Conflicting Tasks}

\author{%
  Srijith Nair \qquad Atilla Eryilmaz \qquad Jia (Kevin) Liu
  \\
  Department of Electrical and Computer Engineering \\
  The Ohio State University \\
  Columbus, OH 43201 \\
  \texttt{\{nair.203, eryilmaz.1, liu.1736\}@osu.edu} \\
}

\iclrfinalcopy 
\begin{document}

\maketitle

\begin{abstract}
    In the literature, many \emph{continual learning} (CL) algorithms have been proposed to address the issue of \emph{catastrophic forgetting} in ML models (i.e., learning new tasks leads to the loss of performance on previously learned tasks).
Although all CL approaches use some form of memory to retain information about past tasks, a grounded understanding of what information needs to be stored to minimize catastrophic forgetting remains elusive.
Recently, it has been recognized that under the strong assumption of the existence of a common global minimizer over all tasks, catastrophic forgetting can be completely avoided.
However, in practice, tasks rarely have a common global minimizer, and a certain amount of forgetting is inevitable.
\change{In this paper, we propose a foundational reframing of CL as balancing conflicting tasks in hindsight from a \emph{multi-task learning} (MTL) perspective.}
The approach is based on finding \emph{Pareto-optimal} solutions, i.e., the solutions which, by definition, minimally forget the previous tasks in the Pareto sense.
\change{We derive a Pareto-minimal-forgetting CL (PMF-CL) algorithm for linear and basis-function regression.
Our algorithm naturally extends to loss functions with quadratic upper-bounds around their minimizers, despite consuming only a \emph{static} memory footprint of $\bigO{d^2}$ for $d$ model parameters, \emph{independent} of the number of sequentially occurring tasks $T$.
In the quadratic-loss setting, we obtain exact hindsight Pareto-optimal guarantees on forgetting.
In the quadratic-upper-bound setting, forgetting is provably bounded, with the tightness of the forgetting bound being directly determined by the tightness of the quadratic bound on the loss function.
Our numerical results validate our theoretical claims of convergence to the Pareto-optimal solution, and demonstrate that we consistently outperform prior empirical methods while occupying lesser memory.}

\end{abstract}

\section{Introduction}
The success of large-scale models for language
\cite{Radford:19:OpenAI,Touvron:23:LLaMA}, vision
\cite{Radford:21:CLIP,Ramesh:22:DALLE2} and multimodal tasks
\cite{OpenAI:24:GPT4,Yin:24:NSR:MLLM} has caused a monumental shift in the way
ML models are trained and deployed.
The clear trend of scaling laws \cite{Kaplan:20:preprint:Scaling} has led to the
AI industry's push towards larger models, larger datasets and more compute for
training.
This has caused the field to frequently hit roadblocks in terms of hardware
and energy requirements \cite{Ahmad:21:JCP:AIEnergy,
Donnell:25:MITTechReview:AIEnergy} to train and fine-tune such models, pushing
active research towards efficiency and adaptation methods.
Likewise, to keep up with the pace of data evolution, efficiently updating
models without completely retraining them or compromising on their performance
on past tasks is equally important.
Thus, there is a pressing need for \emph{continual learning} algorithms, which guarantee high performing models that do not only adapt to new tasks,
but also do so without degrading their performance on past tasks.

Continual Learning (CL) has been studied extensively
\cite{Wang:24:TPAMI:CLSurvey} in the empirical context for at least two decades, where the main motivation is to avoid the so-called \emph{catastrophic forgetting} \cite{French:99:TICS:CatastrophicForgetting,Mermillod:13:FiP:SPT} of past tasks.
Existing CL approaches can be categorized into three subclasses: 1) replay-based, 2) regularization-based, and 3) projection-based methods (see further discussions in \secref{background}).
In terms of geometric intuition, two of these existing methods are relevant:
\textbf{a)} Gradient-projection methods \cite{Saha:21:ICLR:GPM,Lin:22:ICLR:TRGP}
make use of geometric structure of gradient updates on MLPs to constrain them to
separate subspaces for the different tasks;
\textbf{b)} regularization-based methods (e.g., EWC
\cite{Kirkpatrick:17:PNAS:EWC}) use the approximate Hessian matrices of past tasks to nudge the new task updates towards slowly-increasing or non-increasing directions of past task loss functions.
While these works use constraints to reduce or minimize catastrophic forgetting,
they use heuristic methods that either take advantage of null spaces in the parameter space or use approximate regularization.
However, an understanding of \emph{optimality} and \emph{forgetting} in the CL
context from first-principles remains limited.

\begin{figure*}[t]
    \centering
    \begin{subfigure}[valign=T]{0.46\textwidth}
        \centering
        \includegraphics[width=0.95\columnwidth]{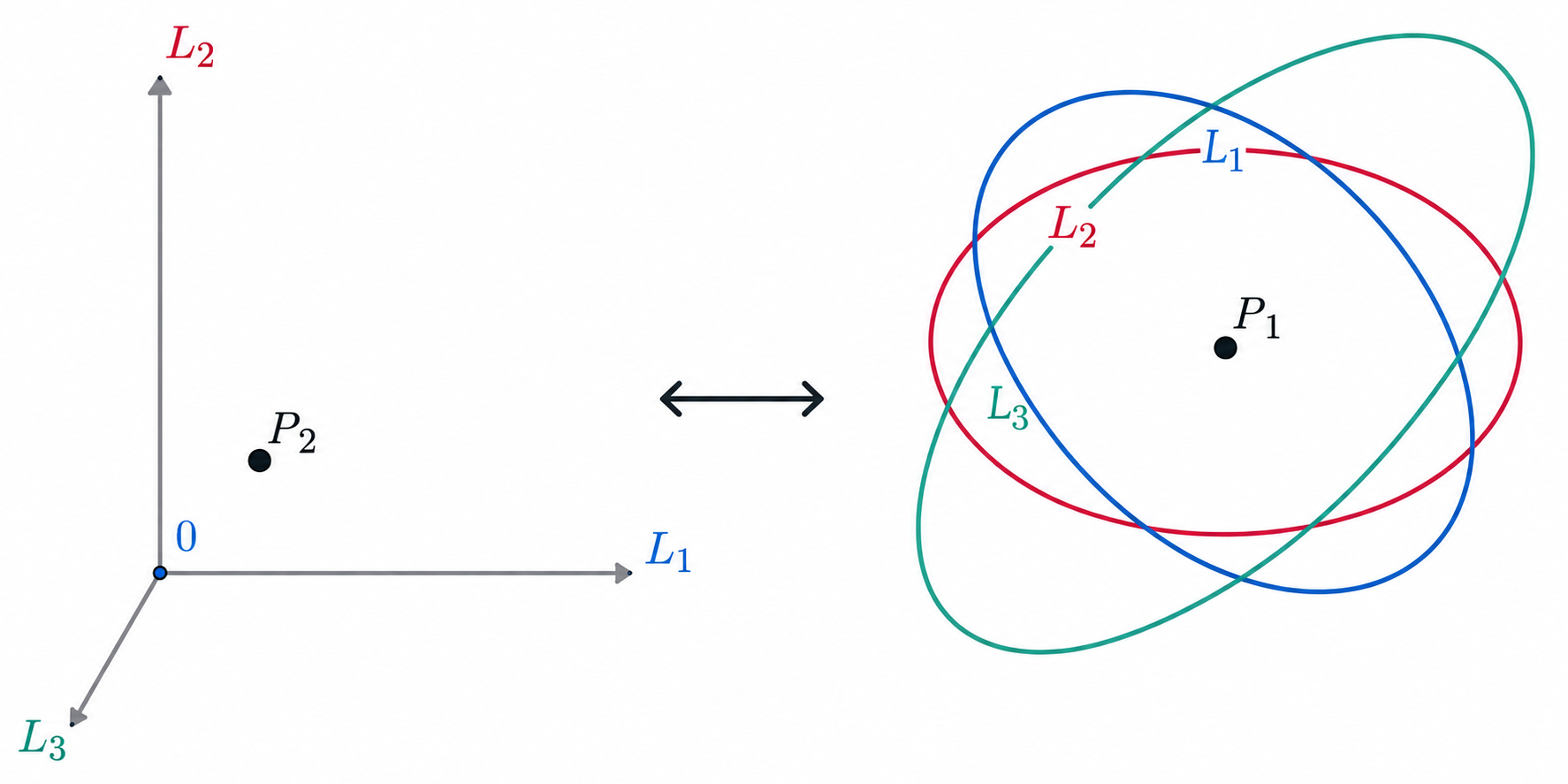}
        \vspace{1.2em}
        \caption{\label{fig:pf-icl}ICL setting \cite{Peng:23:ICML:ICL}}
    \end{subfigure}
    \begin{subfigure}[valign=T]{0.51\textwidth}
        \centering
        \includegraphics[width=0.95\columnwidth]{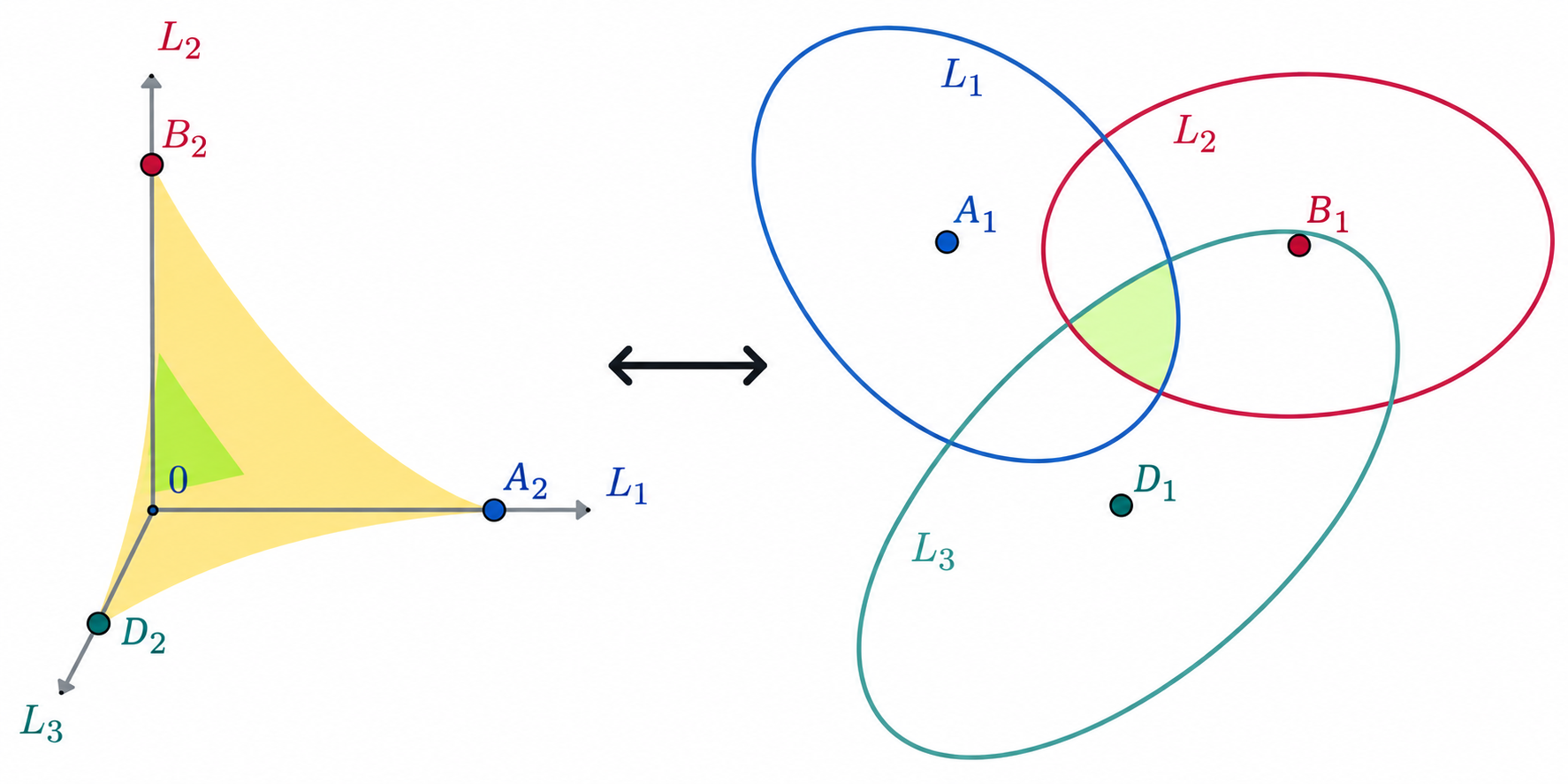}
        \caption{\label{fig:pf-ours}Our PMF-CL setting}
    \end{subfigure}
    \caption{\label{fig:pf-icl-vs-ours}
        Task settings visualized in the task-loss domain \figleft~and parameter
        space \figright.
        In (a) ICL assumes a common global minimizer making the PF collapse to a point in loss domain.
        In (b) the yellow region illustrates a Pareto front (PF).
        The green region depicts a portion of it, and the corresponding set
        of parameters is the intersection of sublevel sets of the loss
        surfaces.}
\vspace{-1.7em}
\end{figure*}

Recently, \citet{Peng:23:ICML:ICL} introduced an interesting approach to develop
an intuitive understanding of forgetting and how representative information from
past tasks impacts forgetting via their proposed \emph{Ideal continual learner}
(ICL).
However, their argument is developed under a strong assumption that the loss
functions of different tasks share at least one common global minimizer.
In most CL scenarios, this assumption would not hold as the
tasks could be in conflict with one another, and the algorithm would inevitably
need to trade-off some tasks.
Relaxing the common-minimizer assumption completely changes the nature of the
problem, since now the CL algorithm needs to find the optimal trade-off point
between various loss functions.
To illustrate this from the Multi-Task Learning (MTL) perspective
\cite{Caruana:97:Springer:MTL, Desideri:12:CRM:MGDA, Hu:22:IEEE:FedMGDA}, the
common-minimizer assumption collapses the \emph{Pareto front} in the task-loss domain (defined formally in \defref{pareto-front}) to a single
point as illustrated in \figref{pf-icl}.
On the other hand, a more realistic scenario is when the tasks have different
minima, and the goal is to find a solution on the Pareto-front as illustrated in
\figref{pf-ours}.

The previous illustration implies that zero-forgetting in general is impossible, and that an ideal CL
algorithm should find the \emph{optimal trade-off} which can minimize all tasks
as much as possible.
The current state-of-the-art in CL still lacks a
fundamental understanding on how to achieve these optimal trade-offs in a
sequential learning setup.
In this work, we develop a first-principle understanding of CL
through the lens of multi-task (or multi-objective) optimization.
To this end, we make the following contributions in this paper:

\vspace{-1em}
\begin{list}{\labelitemi}{\leftmargin=1.5em \itemindent=-0.0em \itemsep=-.2em}
    \item We propose a general meta-algorithmic framework for a \emph{Pareto-minimal-forgetting}
        continual learner (PMF-CL) that can handle conflicting tasks with differing
        minimizers.
        Since our method is developed from the MTL perspective,
        it is readily amenable to geometric intuition in the parameter and
        task-loss function domains.
        Our formulation sheds new insights into exactly what information needs to be memorized to enable Pareto-minimal-forgetting continual learning.
    \item We 
        \change{explicitly} link CL and MTL for general loss functions and
        \change{reframe CL as}
        finding an MTL solution \change{sequentially}. 
        We note that \citet{Peng:23:ICML:ICL} made this connection previously.
        However, they do so under a strong assumption that the tasks share
        atleast one common global minimizer.
    \item For the case of \emph{quadratic} and \emph{\ul{Q}uadratic \ul{U}pper
        \ul{B}ounded} (QUB) loss functions, we also investigate the performances of PMF-CL.
        \change{For the quadratic case, PMF-CL converges to a Pareto-optimal
        solution of all tasks in hindsight.
        For QUB loss, PMF-CL finds a Pareto-optimal solution of the
        surrogate upper-bound functions, and we obtain provable upper-bounds on forgetting.
        The tightness of the forgetting bounds are directly determined by the
        tightness of the corresponding QUBs.}
    \item \change{Additionally, we} 
        explicitly characterize the memory required to achieve
        provably minimal forgetting.
        Interestingly, the memory requirement for quadratic or QUB loss functions is static and bounded
        as $\bigO{d^2}$, where $d$ is the number of parameters.
        Notably, this scaling does not depend on the number of tasks $T$.
        \change{Compared to existing replay-buffer methods, this constant memory footage
        with respect to number of tasks is a significant improvement.}
    \item \change{%
        We perform numerical experiments for continual linear and logistic regression.
        The experiments on linear regression validate our theoretical
        results, by demonstrating that, not only does PMF-CL consistently
        outperform empirical baselines, but it also converges to the
        hindsight-optimal solution, i.e., the solution obtained by jointly
        optimizing over all tasks.
    }
\end{list}

The rest of the paper is organized as follows.
In \secref{background}, we summarize the existing state-of-the-art research in CL.
In \secref{mfcl}, we connect CL to the traditional MTL problem, and present our PMF-CL meta-algorithm for Pareto-minimal forgetting in CL.
In Sections~\ref{sec:pmfcl_regr} and \ref{sec:mfcl_qub}, we design iterative PMF-CL algorithms for regression and QUB problems.
Finally, we conclude in \secref{conclusion}. \label{sec:intro}

\section{Related Work}
Continual Learning (CL) aims to learn from a sequence of tasks while balancing
the trade-off between \emph{plasticity} (learning new data) and \emph{stability}
(retaining old patterns) \cite{Thrun:95:Springer:RobotLL, Mermillod:13:FiP:SPT}.
This paradigm is essential for foundation models in sensitive domains like
healthcare to remain current without suffering \emph{catastrophic
forgetting} \cite{French:99:TICS:CatastrophicForgetting,
Ray:24:BBInf:LLMHealth}. 

Existing empirical strategies generally fall into three broad categories: (a)
\textbf{Experience Replay}, which rehearses samples from previous tasks
\cite{Buzzega:20:NIPS:DER, Aljundi:19:NIPS:MIR}; (b) \textbf{Regularization},
which penalizes parameter deviations in directions critical to past tasks
\cite{Kirkpatrick:17:PNAS:EWC, Zenke:17:ICML:SI}; and (c) \textbf{Gradient
Projection}, which confines updates to orthogonal subspaces to prevent
interference \cite{Saha:21:ICLR:GPM, Lin:22:ICLR:TRGP}. While replay methods are
empirically robust, they often lack the geometric clarity of regularization and
projection approaches.
\change{
Recent work has connected CL with Pareto optimization.
\citet{Lai:25:AAAI:ParetoCL} learn preference-conditioned stability-plasticity
trade-offs, while \citet{Wu:24:ICML:POCL} use Pareto optimization for
buffer-based gradient aggregation across previous tasks.
In contrast, we treat tasks as individual objectives and use loss-preserving
representations to recover a specified hindsight joint-training solution in
structured settings.
}

Recent theoretical efforts have explored CL through the lens of loss landscape
geometry, analyzing how flatness and curvature influence forgetting rates
\cite{Liu:19:preprint:transferability, Benzing:20:preprint:unifying,
Evron:25:preprint:CLSGD}. Although avoiding forgetting is fundamentally NP-hard
\cite{Knoblauch:20:ICML:NPHard}, our proposed method, PMF-CL, leverages
\emph{quadratic upper bounds} to achieve provably Pareto-optimal solutions with
a fixed $O(d^2)$ memory footprint. \appref{lit_review} contains a more
comprehensive literature review.

 \label{sec:background}

\section{Continual Learning as a Multi-Task Optimization Problem} \label{sec:mfcl}
In this section, we prove using first principles that continual learning
(CL) can be viewed as a sequential counterpart of Multi-Task Learning (MTL)
\cite{Caruana:97:Springer:MTL, Miettinen:08:Springer:MOO}.
This connection allows us to derive a general CL meta-algorithm that achieves minimal catastrophic forgetting by finding the same solution as the MTL problem in hindsight.

\paragraph{Notation}
We consider a supervised CL setting with $T$ sequential tasks.
Each task $t$ has an associated dataset of $n_t$ input-output pairs
$\mathcal{D}_t = \{(\vec{x}_i^t, \vec{y}_i^t)\}_{i=1}^{n_t}$
We consider a global model $\vec{f}_{\vec{\theta}}(\cdot): \mathcal{X} \to
\mathcal{Y}$ where $\vec{\theta} \in \mathcal{W}$ is the set of model parameters
of $\vec{f}$.
The loss function for task $t$ is defined as:
\begin{equation}
    L_t(\vec{\theta}) := \mathcal{L}(\vec{\theta}; \mathcal{D}_t)
    = \frac{1}{n_t} \sum_{i=1}^{n_t}
        \ell(\vec{f}_{\vec{\theta}}(\vec{x}_i^t), \vec{y}_i^t)
    \label{eq:task_loss}
\end{equation}
where $\ell: \mathcal{Y} \times \mathcal{Y} \to \mathbb{R}_{\geq 0}$ is a
sample-wise loss function and remains fixed across tasks.
Notably, the model parameters $\vec{\theta}$ and the loss functions
$\mathcal{L}(\cdot;\cdot)$ and $\ell(\cdot,\cdot)$ are shared across all tasks,
with the task-specific dataset $\mathcal{D}_t$ influencing the shape and
properties of each task's loss function.

\subsection{Multi-Task Learning}
\label{sec:mtl}
We begin by defining a few key concepts related to MTL and CL.
\begin{definition}[Multi-task Learning] \label{def:mtl}
    The unconstrained Multi-Task Learning (MTL) problem for $T$ tasks with
    corresponding loss functions $L_t(\theta)$ is defined as:
    $\min_{\theta \in \mathcal{W}} \left\{ 
            L_1(\theta), L_2(\theta), \ldots, L_T(\theta)
        \right\}$
    which is a joint optimization of all potentially conflicting task losses over the shared parameter
    space.
\end{definition}
Since the objectives in MTL can be in conflict with each other, zero-forgetting
is impossible in general, and one would seek a \emph{Pareto optimal} solution or even the entire \emph{Pareto frontier} (i.e., the collection of all Pareto optimal solutions) rather than a common global minimum for all objectives, which does {\em not} exist in general.
This is in contrast to gradient-projection-based approaches in the CL literature
\cite{Saha:21:ICLR:GPM, Lin:22:ICLR:TRGP, Peng:23:ICML:ICL}, which assume a
common global minimizer, or approximate each task's loss function until such an
approximation exists.

Since we will heavily rely on the concept of Pareto optimality, we formally
define these notions next.
\begin{definition}[Pareto Dominance]
    For the MTL problem defined in \defref{mtl}, a solution $\vec{\theta}_a$ is
    said to \emph{Pareto dominate} another solution $\vec{\theta}_b$, denoted as
    $\vec{\theta}_a \succ \vec{\theta}_b$, iff:
    $L_t(\theta_a) \leq L_t(\theta_b), \quad \forall t = 1, 2, \ldots, T$
    and there exists at least one task $t'$ such that:
    $L_{t'}(\theta_a) < L_{t'}(\theta_b).$
\end{definition}

Based on the notion of Pareto dominance, we can now define Pareto optimality as
follows:
\begin{definition}[Pareto Optimality]
    For the MTL problem defined in \defref{mtl}, the Pareto optimal set, denoted
    by $\mathcal{W}^{\star}$, is the set of all parameters in $\mathcal{W}$ that
    are not Pareto dominated by any other parameters in $\mathcal{W}$.
\end{definition}

\begin{definition}[Pareto Front] \label{def:pareto-front}
    The set of all non-dominated solutions in the loss space is called the
    \emph{Pareto Front} (PF):
    $
        \mathcal{P} := \{ (L_1(\vec{\theta}), \dots, L_T(\vec{\theta})):
            \vec{\theta} \in \mathcal{W}^{\star} \}.
    $
\end{definition}

The most common approach to solving MTL problems is \emph{linear scalarization} \cite{Hwang:12:Book:MOO,
Miettinen:08:Springer:MOO}.
For convex problems, it is known that any Pareto optimal solution on the Pareto
front can be found by appropriately choosing a \emph{preference vector}
$\vec{\alpha}^{(T)} \in \Delta^{T}$ from the standard simplex and optimizing the preference-weighted sum of tasks as follows:
\begin{equation}
    \vec{\theta}_{T}^{\star} := \argmin_{\vec{\theta} \in
        \mathcal{W}} \sum_{t=1}^T \alpha_t^{(T)} L_t(\vec{\theta}).
    \label{eq:mtl_lin_scalar}
\end{equation}
Choosing the preference vector $\vec{\alpha}^{(T)}$ in \eqref{mtl_lin_scalar} in
terms of the number of training samples as:
\begin{equation}
    \alpha_t^{(T)} = \frac{n_t}{\sum_{k=1}^T n_k},
    \label{eq:preference_global}
\end{equation}
we can retrieve the global empirical risk minimization (ERM) solution over all
tasks, as stated by the next lemma.
The proof is quite straightforward and is provided in \appref{mtl-global}.
\begin{lemma}[MTL for Global ERM] \label{lem:mtl_global}
    The solution to the linearly-scalarized MTL problem in
    \eqref{mtl_lin_scalar} with preferences chosen as in
    \eqref{preference_global} is equivalent to the solution of the global
    optimization problem $
        \vec{\theta}^{\star} = \argmin_{\vec{\theta} \in \mathcal{W}}
        \mathcal{L}(\vec{\theta}; \mathcal{D}),
    $
    where $\mathcal{D}$ is a concatenation of all datasets $\{ \mathcal{D}_1,
    \dots, \mathcal{D}_T \}$.
\end{lemma}
\begin{remark}
    In the scenario where tasks repeat over time, the preference vector in
    \eqref{preference_global} automatically weights the repeated tasks more,
    which is arguably desirable, and is analogous to having repeated
    samples/batches in the global dataset $\mathcal{D}$.
\end{remark}

\subsection{Continual Learning as Sequential MTL}
\label{sec:cl_as_mtl}
The key difference between MTL and CL is in the availability of the data.
In MTL, data from all tasks are available simultaneously, so the goal is to find a Pareto optimal solution.
In CL, there are two key differences from MTL:
\textbf{1)} data from tasks are available \emph{sequentially}, one at a time.
\textbf{2)} the model must \emph{evolve} after each task, without forgetting previous tasks.
Thus, a CL algorithm must maintain some \emph{memory} of older tasks during optimization.
Here, we derive a general meta-algorithmic framework called \emph{\ul{P}areto-\ul{M}inimal-\ul{F}orgetting \ul{C}ontinual \ul{L}earner} (PMF-CL).

First, we address the question:
\change{\emph{What is the \ul{sufficient amount of memory} required to preserve
the past task loss functions, enough to retrieve a Pareto-optimal trade-off for
all tasks in hindsight?}}
Toward this end, we first concretely define `forgetting':
\begin{definition}[Forgetting] \label{def:forgetting}
    We define the \emph{degree of forgetting} of the $\nth{t}$ task by a
    parameter $\vec{\theta} \in \mathcal{W}$ in terms of the loss function
    $L_t(\vec{\theta})$ as $
        F_t(\vec{\theta}) := L_t(\vec{\theta}) - \inf_{\vec{\theta}' \in
            \mathcal{W}} L_t(\vec{\theta}'),$
    and the \emph{average forgetting} for the first $k$ tasks as $\overline{F}_{k}(\vec{\theta}) :=
            \frac{1}{k} \sum_{t=1}^k F_t(\vec{\theta}).$
\end{definition}
\change{Since by definition $F_t(\vec{\theta})$ is an objective-wise constant
shift, the Pareto-optimality of $F_t(\vec{\theta})$ and $L_t(\vec{\theta})$ are
equivalent.}
The task loss function defined in \eqref{task_loss} is a
function of $\vec{\theta}$ and fully characterized by the dataset
$\mathcal{D}_t$.
Consider an equivalent notation of the task loss in terms of the
\emph{\ul{M}inimal \ul{S}ufficient \ul{I}nformation} (MSI), $\mathcal{I}_t$, required to
fully characterize $\mathcal{L}(\cdot, \cdot)$.
We formally define the MSI as:

\begin{definition}[Minimal Sufficient Information; MSI] \label{def:msi}
    For a given model $\vec{f}_{\vec{\theta}}(\cdot)$ and sample loss $\ell(\cdot, \cdot)$, the \emph{MSI} for task $t$, $\mathcal{I}_t$ is defined as the smallest set that can fully characterize the loss function of the task.
    Formally, $\mathcal{I}_t$ satisfies $0 < |\mathcal{I}_t| \leq |\mathcal{D}_t|$ and
    $\mathcal{\tilde{L}}(\vec{\theta}; \mathcal{I}_t) = \mathcal{L}(\vec{\theta}; \mathcal{D}_t) + C_t,
        \forall~\vec{\theta} \in \mathcal{W}$ and some constant $C_t$;
    and there exists no set $\mathcal{I}_t'$ such that $|\mathcal{I}_t'| < |\mathcal{I}_t|$ and
    $ \mathcal{\tilde{L}}(\vec{\theta}; \mathcal{I}_t')
        = \mathcal{L}(\vec{\theta}; \mathcal{D}_t),~
        \forall~\vec{\theta} \in \mathcal{W}$.
\end{definition}

The minimal task information $\mathcal{I}_t$ is fully determined by the
shape and properties of the composition of the model and the loss function,
$\ell(f_{\vec{\theta}}(\cdot), \cdot)$.
In the following section, we show that for loss functions with a quadratic structure, $\mathcal{I}_t$ can be compact (such as the \emph{right singular matrix} and the \emph{singular values} of the dataset).
In the worst case where the loss surfaces have a complex dependency on the data, the minimal information might require the entire dataset $\mathcal{D}_t$.
\change{Thus, how to automatically identify MSI for arbitrary loss functions remains an open challenge.
In fact, \citet{Knoblauch:20:ICML:NPHard} already established that this problem is NP-Hard in general.
However, for certain structured cases, such as the quadratic case discussed later in this paper, it is indeed possible to exploit their special structure to
find the loss-preserving minimal sufficient information.}

Based on the notion of MSI, we can propose the PMF-CL meta-algorithm
described in \algref{general_mfcl}, which would \emph{Pareto-minimally forget}, or optimally trade-off, previous tasks by finding the same solution as MTL in a
sequential manner.
\change{%
\algref{general_mfcl} is meant to serve as a general conceptual link between
continual learning and sequential MTL.
In general, it is not computationally tractable without access to the structure
of the loss function.}
Importantly, note that the ordering of tasks has no impact on the final solution
learned, since the Pareto-optimal solution given a fixed preference vector
$\alpha^{(k)}$ is not affected by the task order.
\change{This is an intentional choice, since we target to arrive at the
same hindsight-optimal solution regardless of the order in which the tasks may
be presented.}
In Sections~\ref{sec:pmfcl_regr} and \ref{sec:mfcl_qub} we will show
how the general idea presented in \algref{general_mfcl} can be specialized for loss functions with quadratic structure.

\begin{table}[t!]
\begin{minipage}[t]{0.48\textwidth}
    \begin{algorithm}[H]
    \caption{\label{alg:general_mfcl}Generic PMF-CL Meta-algorithm}
    \begin{algorithmic}[1]
        \REQUIRE{Loss functions $\mathcal{L}(\cdot; \cdot)$ and $\mathcal{\tilde{L}}(\cdot; \cdot)$,
            Tasks $\{\mathcal{D}_t\}_{t=1}^T$}
        \FOR{$t \gets 1, \dots, T$}
            \STATE \textbf{Receive} task dataset $\mathcal{D}_t$
            \STATE \textbf{Solve} the MTL problem: \alglinelabel{mtl-problem}
                {\noindent\small $$
                \min_{\vec{\theta} \in \mathcal{W}} \{
                    \mathcal{\tilde{L}}(\vec{\theta}; \mathcal{I}_1),
                    \closedots,
                    \mathcal{\tilde{L}}(\vec{\theta}; \mathcal{I}_{t-1}),
                    \mathcal{L}(\vec{\theta}; \mathcal{D}_{t})
                \}$$}
                in an incremental fashion.
            \STATE \textbf{Store} minimal information for task $t$:
                $\mathcal{I}_t$
        \ENDFOR
    \end{algorithmic}
    \end{algorithm}
    \vspace{-1.5em}
    
    \begin{algorithm}[H]
    \caption{\label{alg:mfcl_qub}PMF-CL for QUB loss (\secref{mfcl_qub})}
    \begin{algorithmic}[1]
        \REQUIRE{QUB Loss $\mathcal{L}(\cdot; \cdot)$, Tasks $\{\mathcal{D}_t\}_{t=1}^T$}
        \STATE $\vec{\theta}_1^{\star} = \vec{\theta}_1^{\min}
            = \argmin_{\vec{\theta} \in \mathcal{W}} \mathcal{L}(\vec{\theta};
            \mathcal{D}_1)$
        \STATE Initialize $\vec{A}_1 = \vec{H}_1$
        \FORALL{$t \gets 2, \dots, T$}
            \STATE Compute $\vec{\theta}_t^{\min} =
                \argmin_{\vec{\theta} \in \mathcal{W}} \mathcal{L}(\vec{\theta};
                \mathcal{D}_t)$.
            \STATE Determine $\vec{H}_t$, the Hessian upper bound.
            \STATE $\vec{A}_t = \frac{N_{t-1}}{N_t} \vec{A}_{t-1} + \alpha_t^{(t)} \vec{H}_t$
            \STATE Solve for $\Delta \vec{\theta}$:
                $ \vec{A}_t \Delta \vec{\theta} = \alpha_t^{(t)} \vec{H}_t
                (\vec{\theta}_t^{\min} - \vec{\theta}_{t-1}^{\star})$
            \STATE Update $\vec{\theta}_{t}^{\star} =
                \vec{\theta}_{t-1}^{\star} + \Delta \vec{\theta}$
        \ENDFOR
    \end{algorithmic}
    \end{algorithm}
\end{minipage}%
\hfill 
\begin{minipage}[t]{0.48\textwidth}
\begin{algorithm}[H]
    \caption{\label{alg:linreg-cl}PMF-CL for Linear (\emph{Option I}) and Basis
    (\emph{Option II}) Regression.}
    \begin{algorithmic}[1]
        \REQUIRE Datasets $(\vec{X}_t, \vec{y}_t)_{t=1}^T$,
            basis function $\phi(\cdot)$.
        \STATE Init $t = 1$
        \WHILE{$t < T$ \textbf{and} $\sum_{i=1}^t r_i \leq d$}
            \STATE \emph{Option I:} $\vec{U}_t^e, \vec{\Sigma}_t^e, \vec{V}_t^e =
                \mathrm{eSVD}(\vec{X}_t)$
            \\ \emph{Option II:} $\vec{U}_t^e, \vec{\Sigma}_t^e, \vec{V}_t^e =
                \mathrm{eSVD}(\phi(\vec{X}_t))$
            \STATE $\vec{\tilde{y}}_{t} = (\vec{\Sigma}_t^e)^{-1}
                {\vec{U}_t^e}^{\top} \vec{y}_t$
            \STATE Store $(\vec{\Sigma}_t^e, \vec{V}_t^e, \vec{\tilde{y}}_{t})$
                \COMMENT{{\footnotesize Mem: $(2 + d) \sum_{i=1}^t r_i$}}
            \STATE Compute preferences $\vec{\alpha}^{(k)}$ as in
                \eqref{preference_global}
            \STATE Solve for $\vec{\theta}_t^{\star}$ using \eqref{theta-p-star}
            \STATE Increment $t \gets t + 1$
        \ENDWHILE
        \STATE Init $\vec{A}_{t-1} = \sum_{i=1}^{t-1} \alpha_i^{(t-1)} \vec{V}_i^e
            {\vec{\Sigma}_i^e}^2 {\vec{V}_i^e}^{\top}$
        \FORALL{$k \gets t, \dots, T$}
            \STATE \emph{Option I:} $\vec{U}_k^e, \vec{\Sigma}_k^e, \vec{V}_k^e =
                \mathrm{eSVD}(\vec{X}_t)$
            \\ \emph{Option II:} $\vec{U}_k^e, \vec{\Sigma}_k^e, \vec{V}_k^e =
                \mathrm{eSVD}(\phi(\vec{X}_t))$
            \STATE $\vec{A}_k = \frac{N_{k-1}}{N_k} \vec{A}_{k-1} + \alpha_k^{(k)} \vec{V}_k^e
                {\vec{\Sigma}_k^e}^2 {\vec{V}_k^e}^{\top}$
            \STATE $\vec{\tilde{y}}_k = (\vec{\Sigma}_k^e)^{-1}
                {\vec{U}_k^e}^{\top} \vec{y}_k$
            \STATE Solve for $\Delta\vec{\theta}_k$ as in \eqref{iter-update}
            \STATE Update $\vec{\theta}_k^{\star} = \vec{\theta}_{k-1}^{\star} +
                \Delta \vec{\theta}_k$
        \ENDFOR
        \ENSURE $\{\vec{\theta}_k^{\star}\}_{k=1}^T$
    \end{algorithmic}
\end{algorithm}
\end{minipage}

\end{table}


\section{PMF-CL for Regression} \label{sec:pmfcl_regr}
  In this section, we will study the specialization of the meta-algorithm in \algref{general_mfcl}
to the problem of linear and basis-function regression.
Specifically, we will derive the MSI, and then design a memory-efficient and
provably Pareto-minimal-forgetting, continual learning algorithm for linear
regression in \secref{mfcl_linreg} and its extension to basis functions in
\appref{mfcl_basis}.

\subsection{Linear Regression} \label{sec:mfcl_linreg}

In linear regression, the model is given by $f_{\vec{\theta}}(\vec{x}) = \vec{\theta}^\top
\vec{x}$ where $\vec{\theta} \in \Real^{d \times 1}$ and the output is a scalar.
Under the mean squared error (MSE) loss function, task loss $L_t$ can be written as:
\begin{equation}
    L_t(\vec{\theta}) = \mathcal{L}(\vec{\theta}; \mathcal{D}_t)
        = \frac{1}{2 n_t} \norm{\vec{X}_t \vec{\theta} - \vec{y}_t}_2^2,
    \label{eq:linreg-task-loss}
\end{equation}
where $\vec{X}_t \in \Real^{n_t \times d}$ (resp. $\vec{Y}_t \in \Real^{n_t \times
1}$) is a matrix with the inputs $\vec{x}_{t,i}$ as rows (resp. vector of outputs
$y_{t,i}$).
In order to obtain the MSI of the quadratic loss function, we characterize the task data matrix $\vec{X}_t$ via the \emph{singular value decomposition} (SVD).

\textbf{SVD and economy SVD. }
For the task data matrix $\vec{X}_t$, we define the SVD in two ways:
$\vec{X}_t = \vec{U}_t \vec{\Sigma}_t \vec{V}_t^{\top} = \vec{U}_t^e
        \vec{\Sigma}_t^e {\vec{V}_t^e}^{\top},$
where $\vec{U}_t \in \Real^{n_t \times n_t}$ and $\vec{V}_t \in \Real^{d \times
d}$ are orthonormal matrices, i.e., $\vec{V}^{\top}\vec{V} = \vec{V}
\vec{V}^{\top} = \vec{I}_{d}$, and $\vec{U}^{\top}\vec{U} = \vec{U}
\vec{U}^{\top} = \vec{I}_{n_t}$, and the singular values are the
diagonal elements of $\vec{\Sigma}_t \in \Real^{n_t \times d}$.
The \emph{economy SVD} (eSVD) stores only the singular vectors corresponding to the
non-zero singular values in $\Sigma_t$.
Suppose that the rank of $\vec{X}_t$ is $r_t$.
Then, the matrices $\vec{U}_t^e \in \Real^{n_t \times r_t}$ and $\vec{V}_t^e \in
\Real^{d \times r_t}$ can be truncated to contain the first $r_t$ columns of
$\vec{U}_t$ and $\vec{V}_t$, respectively, and $\vec{\Sigma}_t^e \in \Real^{r_t
\times r_t}$ is a diagonal matrix of positive reals.

\begin{theorem}[MSI for Linear Regression] \label{thm:msi_linear}
    We can write the loss function for linear regression as:
    \begin{equation}
        L_t(\vec{\theta}) = \frac{1}{2n_t} \norm{{\vec{V}_t^e}^{\top}
            \vec{\theta} - \tilde{\vec{y}}_t}_{\vec{\Sigma}_t^e}^2
        + \min_{\vec{\theta}' \in \mathcal{W}} L_t(\vec{\theta}'),
    \end{equation}
    where we use the notation $\norm{\vec{z}}_{\vec{A}}^2 :=
    \norm{\vec{A} \vec{z}}_2^2$ and $\vec{\tilde{y}} = {\vec{\Sigma}_t^e}^{-1}
    \vec{U}_t^e \vec{y}_t$.
    Then, \change{an MSI representing the linear regression loss is
    $\mathcal{I}_t^{\mathrm{lin}} := (\mathcal{S}(\vec{V}_t^e), \vec{\Sigma_t^e},
    \vec{\tilde{y}}_t)$, where $\mathcal{S}(\cdot)$ is a compact storage
    model for the orthonormal singular-vector matrix $\vec{V}_t^e$ using Givens
    rotations \cite{Shepard:15:JPCA:Givens}.
    The memory size for this representation is
    $\mathrm{size}(\mathcal{I}_t^{\mathrm{lin}}) = d r_t - \frac{r_t (r_t -
    3)}{2}$ which is provably \emph{minimal}.}
\end{theorem}
A complete proof of \thmref{msi_linear} is deferred to \appref{msi_linear}.
Since the last term in the RHS does not depend on $\vec{\theta}$, PMF-CL does not need to remember it.
Note that any representative set smaller than $\mathcal{I}_t^{\mathrm{lin}}$
would inevitably forget the loss function shape, leading to forgetting in a CL
setting.
Thus, $\mathcal{I}_t^{\mathrm{lin}}$ is indeed the \emph{minimal} information
required to remember the $\nth{t}$ task.
Considering the CL problem up to task $k$, a Pareto optimal solution
$\vec{\theta}_{k}^{\star}$ to the multi-task problem in \eqref{mtl_lin_scalar}
under linear scalarization with the preference vector $\vec{\alpha}^{(k)} :=
[\alpha_1^{(k)}, \alpha_2^{(k)}, \dots, \alpha_k^{(k)}]^{\top}$ can be written
as:
$\vec{\theta}_{k}^{\star} = \argmin_{\vec{\theta}} \sum_{t=1}^k
        \frac{\alpha_t^{(k)}}{2 n_t} \norm{{\vec{V}_t^e}^{\top} \vec{\theta} -
            \vec{\tilde{y}}_t}_{\vec{\Sigma}_t^e}^2,$
which can be solved by the linear system:
\begin{equation}
    \!\!\!\!\!\! \left(
        \sum_{t=1}^k \frac{\alpha_t^{(k)}}{n_t} \vec{V}_t^e {\vec{\Sigma}_t^e}^{2}
        {\vec{V}_t^e}^{\top}
    \right) \vec{\theta}_{k}^{\star} = \left(
        \sum_{t=1}^k \frac{\alpha_t^{(k)}}{n_t} \vec{V}_t^e {\vec{\Sigma}_t^e}^{2}
        \vec{\tilde{y}}_t 
    \right). \!\!\!
    \label{eq:theta-p-star}
\end{equation}
Thus, one would require a memory of $(d + 2) \sum_{t=1}^T r_t$ floating-points per task to store $(\vec{V}_t^e$, $\vec{\Sigma}_t^e, \vec{\tilde{y}}_t)$.

\ul{\bf Memory-efficient Continual Linear Regression:}
Armed with the MSI of \thmref{msi_linear}, one can choose to directly apply
the PMF-CL meta-algorithm in \algref{general_mfcl} using \eqref{theta-p-star}.
However, due to the quadratic nature of task loss functions, we can optimize
memory requirements further through iterative updates on the parameters.
As we show next, this helps us to limit memory requirements to $\bigO{d^2}$
{\em irrespective of} the number of tasks.
We choose preferences $\vec{\alpha}^{(k)}$ as in \eqref{preference_global}.
Then, the following result, proven in \appref{mfcl_linear}, holds for the Pareto optimal solution up to task $k$:
\begin{theorem}[Iterative PMF-CL for Linear Regression] \label{thm:mfcl_linear}
    Using the MSI of the task loss as $\mathcal{I}_t^{\mathrm{lin}} :=
    (\vec{V}_t^e, \vec{\Sigma}_t^e, \vec{\tilde{y}}_t)$ from \thmref{msi_linear},
    the PMF-CL update at the $\nth{k}$ task can be written as the solution to:
    \begin{equation}
        \vec{A}_k \Delta \vec{\theta}_k =
            \alpha_k^{(k)} \vec{V}_k^e {\vec{\Sigma}_k^e}^2
            (\vec{\tilde{y}}_k - {\vec{V}_k^e}^{\top}
            \vec{\theta}_{k-1}^{\star}),\qquad k \in [T]
        \label{eq:iter-update}
    \end{equation}
    Here, $\vec{A}_0 \!=\! \vec{0}, \vec{A}_k \!=\! \frac{N_{k-1}}{N_k} \vec{A}_{k-1} \!+\! \alpha_k^{(k)} \vec{V}_k^e {\vec{\Sigma}_k^e}^2 {\vec{V}_k^e}^{\top}$, $N_k \!:=\! \sum_{i=1}^k n_i$, and $\Delta \vec{\theta}_k \!=\! \vec{\theta}_k^{\star} \!-\! \vec{\theta}_{k-1}^{\star}$.
\end{theorem}

In \eqref{iter-update}, we have a linear system of equations to solve for the {\em incremental update} $\Delta \vec{\theta}_k$, where we need to store the matrices $\vec{A}_k$ of size
$d^2$, $(\vec{\nu}_k - M_k \vec{\theta}_{k-1}^{\star})$ of size $d$, and
$\vec{\theta}_{k-1}^{\star}$ of size $d$.
From the memory analysis of both procedures \eqref{theta-p-star} and
\eqref{iter-update}, we can initially start by performing the direct computation
in \eqref{theta-p-star} until the memory required to store
$\mathcal{I}_t^{\mathrm{lin}}$ for all $t$, exceeds that required by
\eqref{iter-update}, and then switch to iterative update.
This switching threshold is given by
$(d + 2) \sum_{t=1}^k r_t \geq (d + 2) d$ or $\sum_{t=1}^k r_t \geq d$.
The PMF-CL algorithm for linear regression is summarized in \algref{linreg-cl},
where preferences $\vec{\alpha}$ are calculated according to
\eqref{preference_global}.
The total memory complexity of \algref{linreg-cl} is $\bigO{d^2}$, and remains {\em constant} with respect to the number of
tasks $T$ once the above rank-threshold is reached.

\change{With the sample-proportional preference vector in
\eqref{preference_global}, the recursive update in \algref{linreg-cl} becomes
the block, information-form recursive least-squares computation for ERM
\cite{Plackett:50:Biometrika:RLS,Jiang:04:Elsevier:RLS}.
Our contribution is the Pareto formulation and its CL interpretation via the
loss-preserving MSI, which recovers the Plackett recursive update as special cases.}
PMF-CL for linear regression naturally extends to models that are linear in
their parameters, including basis functions and two-layer MLPs under the NTK
regime (see \appref{mfcl_basis} for full derivations).

\subsection{Forgetting in Regression} \label{sec:mfcl_forgetting_linreg}
From \algref{linreg-cl}, we can derive an exact expression of forgetting as defined
in \defref{forgetting}.
Before we state the forgetting result, we give an intuitive understanding on
the geometry of the Pareto front (PF) for the loss functions $L_t(\cdot)$ and the corresponding
forgetting functions $F_t(\cdot)$.
    Note that the PFs of quadratic loss functions form a convex surface in the task-loss domain, as shown in \figref{pf-ours}.
    After the $\nth{k}$ task, the degree of forgetting, as defined in
    \defref{forgetting} for all previous tasks including $k$, would be a
    identically shaped PF but shifted so that it touches all the coordinate
    hyperplanes, i.e., hyperplanes $F_1 = 0$, $F_2 = 0$, \dots, $F_k = 0$ of
    dimension $k-1$ atleast once.

The next result (proof in \appref{forgetting-linear}) characterizes
forgetting in linear and basis-function regression:
\begin{theorem} \label{thm:linreg_forgetting}
    The forgetting of the $\nth{t}$ task after learning $k$ tasks using
    \algref{linreg-cl} is given by:
    \begin{equation}
        F_t(\vec{\theta}) = \frac{1}{2 n_t} \norm{
            {\vec{V}_t^e}^{\top} \vec{A}_k^{\dagger} \vec{\nu}_k - \vec{\tilde{y}}_t
        }^2,
    \end{equation}
    where $\vec{A}_k$ and $\vec{\nu}_k$ are defined as
        $\vec{A}_k := \sum_{t=1}^k \alpha_t^{(k)} \vec{V}_t^e {\vec{\Sigma}_t^e}^2
            {\vec{V}_t^e}^{\top},
        ~~
        \vec{\nu}_k := \sum_{t=1}^k \alpha_t^{(k)} \vec{V}_t^e {\vec{\Sigma}_t^e}^2
            \vec{\tilde{y}}_t,$
    and $\dagger$ represents the Moore-Penrose pseudo-inverse.
\end{theorem}
From \thmref{linreg_forgetting}, the forgetting $\{F_t(\cdot)\}_{t=1}^k$ of all tasks is Pareto optimal, because the PMF-CL algorithm finds a solution on the PF of the corresponding loss functions up to task $k$, by design.
In order to empirically validate our theoretical claims, we perform experiments on synthetic linear regression datasets in \appref{experiments}.
Our experiments simulate different optimal values for each task, thus adhering to more general settings than just tasks with common global minimizers.

\section{Quadratic Upper Bounded Loss Functions} \label{sec:mfcl_qub}

In general, it may not be easy to derive a specialized PMF-CL algorithm for any given loss functions, since the MSI $\mathcal{I}_t$ could be hard to determine.
However, we can still obtain an upper-bound on the degree of forgetting for
competing tasks, if they satisfy the \emph{quadratic upper bound} (QUB)
property, which we formally state as follows.

\begin{assumption}[Quadratic Upper Bound] \label{asm:qub}
    We assume that the loss function function $\mathcal{L}(\vec{\theta};
    \mathcal{D}_t)$ has a quadratic upper bound (QUB), i.e., for some p.s.d. Hessian matrix $\vec{H}_t$ and task minimizer
    $\vec{\theta}_t^{\min} := \argmin_{\vec{\theta} \in \mathcal{W}}
    \mathcal{L}(\vec{\theta}; \mathcal{D}_t)$, it holds that
    $L_t(\vec{\theta}) \leq L_t^{\mathrm{qub}}(\vec{\theta})$ with $L_t^{\mathrm{qub}}$ given by:
    \begin{equation}
        L_t^{\mathrm{qub}}(\vec{\theta}) := L_t(\vec{\theta}_t^{\min})
            + \frac{1}{2} (\vec{\theta} - \vec{\theta}_t^{\min})^T \vec{H}_t
                (\vec{\theta} - \vec{\theta}_t^{\min}).
    \label{eq:qub_loss_ub}
    \end{equation}
\end{assumption}

To put \asmref{qub} into perspective, we note that this assumption can be satisfied by all smooth (gradient-Lipschitz)
functions if $\vec{H}_t$ is positive definite, i.e., for some $h > 0$ and
$\forall \vec{\theta}, \vec{\theta}' \in \mathcal{W}$, functions satisfying:
$L_t(\vec{\theta}) \leq L_t(\vec{\theta}')
        + \nabla L_t(\vec{\theta}')^T (\vec{\theta} - \vec{\theta}')
        + \frac{h}{2} \|\vec{\theta} - \vec{\theta}'\|_2^2$.
However, \asmref{qub} allows the local curvature of the loss function around its
minimum to vary by using a p.s.d. matrix $\vec{H}_t$, rather than a single positive
scalar $h > 0$.

For loss functions satisfying \asmref{qub}, we can derive a PMF-CL
algorithm by computing the Pareto solution of the QUB functions.
The tightness of the upper-bound \eqref{qub_loss_ub} directly determines the
degree of forgetting of the CL algorithm.
Note that Eq.~\eqref{qub_loss_ub} is very similar in form to the quadratic loss from
the linear regression case.
Then, we compute the Pareto solution to the MTL problem,
$
    \min_{\vec{\theta} \in \Theta}
        \{
            L_1^{\mathrm{qub}}(\vec{\theta}), L_2^{\mathrm{qub}}(\vec{\theta}),
            \dots, L_k^{\mathrm{qub}}(\vec{\theta})
        \},
$
using linear scalarization.
Using a preference $\vec{\alpha}^{(k)}$ as in \eqref{mtl_lin_scalar}, we
can compute the Pareto optimal solution $\vec{\theta}_{k}^{\star}$ up to task $k$
by solving the linear equation:
\begin{equation}
    \left(\sum_{t=1}^k \alpha_t^{(k)} \vec{H}_t \right)
        \vec{\theta}_{k}^{\star} = \sum_{t=1}^k \alpha_t^{(k)} \vec{H}_t
        \vec{\theta}_t^{\min}.
    \label{eq:mfcl_qub_solution}
\end{equation}
As in \secref{mfcl_linreg}, the quadratic nature of the problem allows us
to derive a memory-efficient iterative update as in \thmref{mfcl_qub} (the detailed proof is relegated to \appref{mfcl_qub}).
We present the PMF-CL pseudocode specialized for tasks satisfying QUB in \algref{mfcl_qub}.

\begin{theorem}[Iterative PMF-CL for QUB Loss] \label{thm:mfcl_qub}
    \change{Denoting the QUB-surrogate-preserving information} \footnote{The
    term \emph{minimal sufficient information} is not used here, since
    technically $(\vec{H}_t, \vec{\theta}_{\min})$ is an MSI of the QUB
    surrogate function and not the actual loss function.
    The actual loss function would require a potentially larger (or smaller
    depending on the tasks considered) memory to represent the full task loss
    faithfully.} of the task loss as $\mathcal{I}_t = \{\vec{H}_t,
    \vec{\theta}_t^{\min}\}$ and using $\vec{\alpha}^{(k)}$ from
    \eqref{preference_global}, the incremental PMF-CL update specialized for
    tasks with QUB losses can be written as:
    \begin{equation}
        \vec{A}_k \Delta \vec{\theta}_k =
            \alpha_k^{(k)} \vec{H}_k
            (\vec{\theta}_k^{\min} - \vec{\theta}_{k-1}^{\star}),
        \label{eq:pmfcl_qub_ak_update}
    \end{equation}
    where $\vec{A}_0 = 0, \vec{A}_k = \frac{N_{k-1}}{N_k} \vec{A}_{k-1} +
    \alpha_k^{(k)} \vec{H}_k$, $\vec{\theta}_k^{\min} := \min_{\vec{\theta}} L_k(\vec{\theta})$, and $N_k = \sum_{i=1}^k n_i$.
\end{theorem}

\change{We note that, in deep learning, t}he Hessian
matrix $\vec{H}_t \in \Real^{d \times d}$ could be large, and \change{the
$\bigO{d^2}$ scaling can be expensive}.
However, similar to regularization methods
\cite{Kirkpatrick:17:PNAS:EWC} for log-likelihood loss, one can use
the Monte-Carlo approximation of the Fisher information matrix (FIM) as a
Hessian estimate.
The FIM requires only storing first-order derivatives, thus greatly reducing
memory usage from $\bigO{d^2}$ to $\bigO{d}$.


\begin{figure*}[t]
    \centering
    \begin{subfigure}[t]{0.48\textwidth}
        \centering
        \includegraphics[width=0.8\columnwidth]{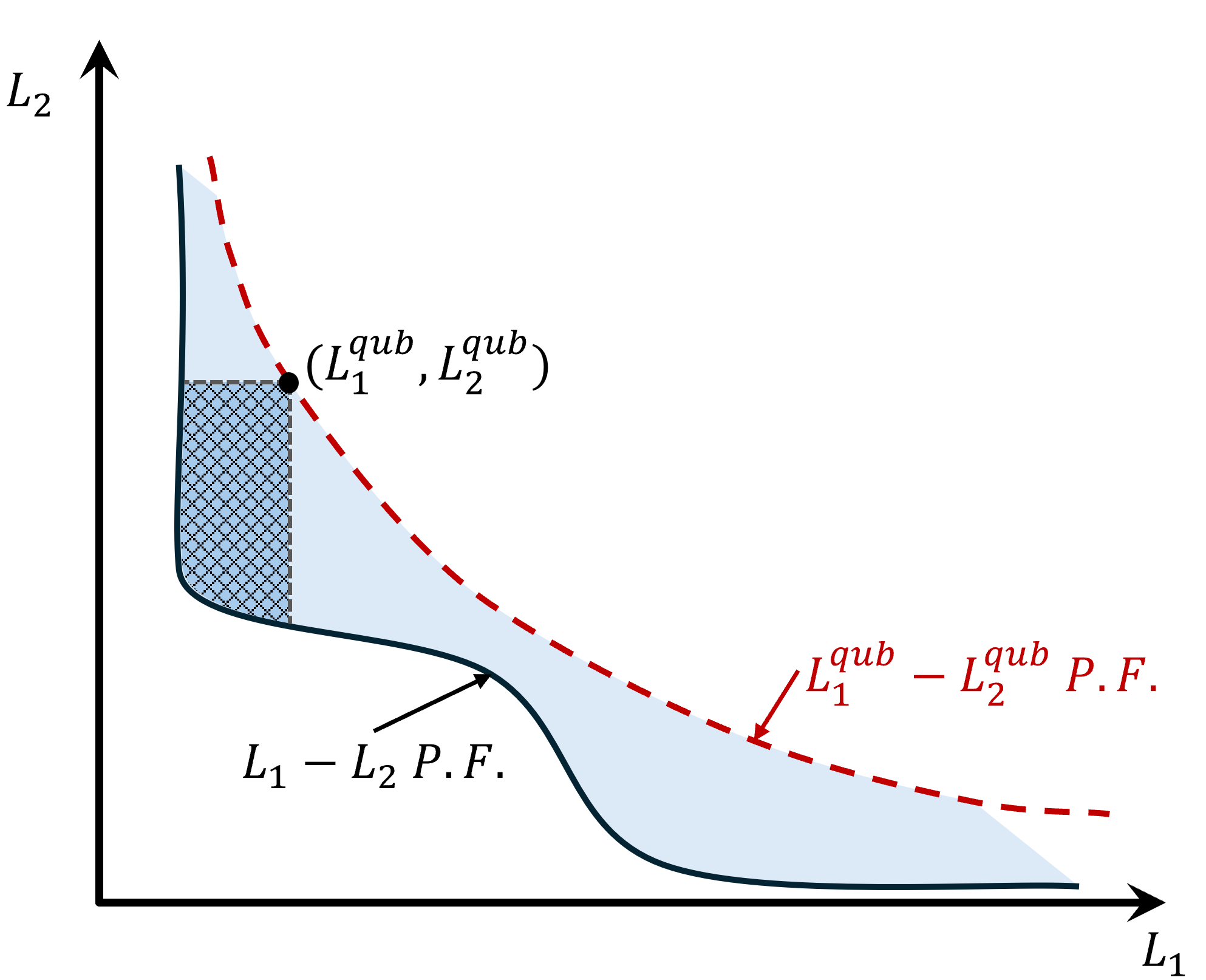}
        \caption{\label{fig:qub_pf}
            \textbf{Task-loss domain:} A Pareto solution marked at $(L_1^{\mathrm{qub}},
            L_2^{\mathrm{qub}})$; the hatch depicts the region in
            which the respective true loss functions might lie.}
    \end{subfigure}
    \hfill
    \begin{subfigure}[t]{0.48\textwidth}
        \centering
        \includegraphics[width=1.0\columnwidth]{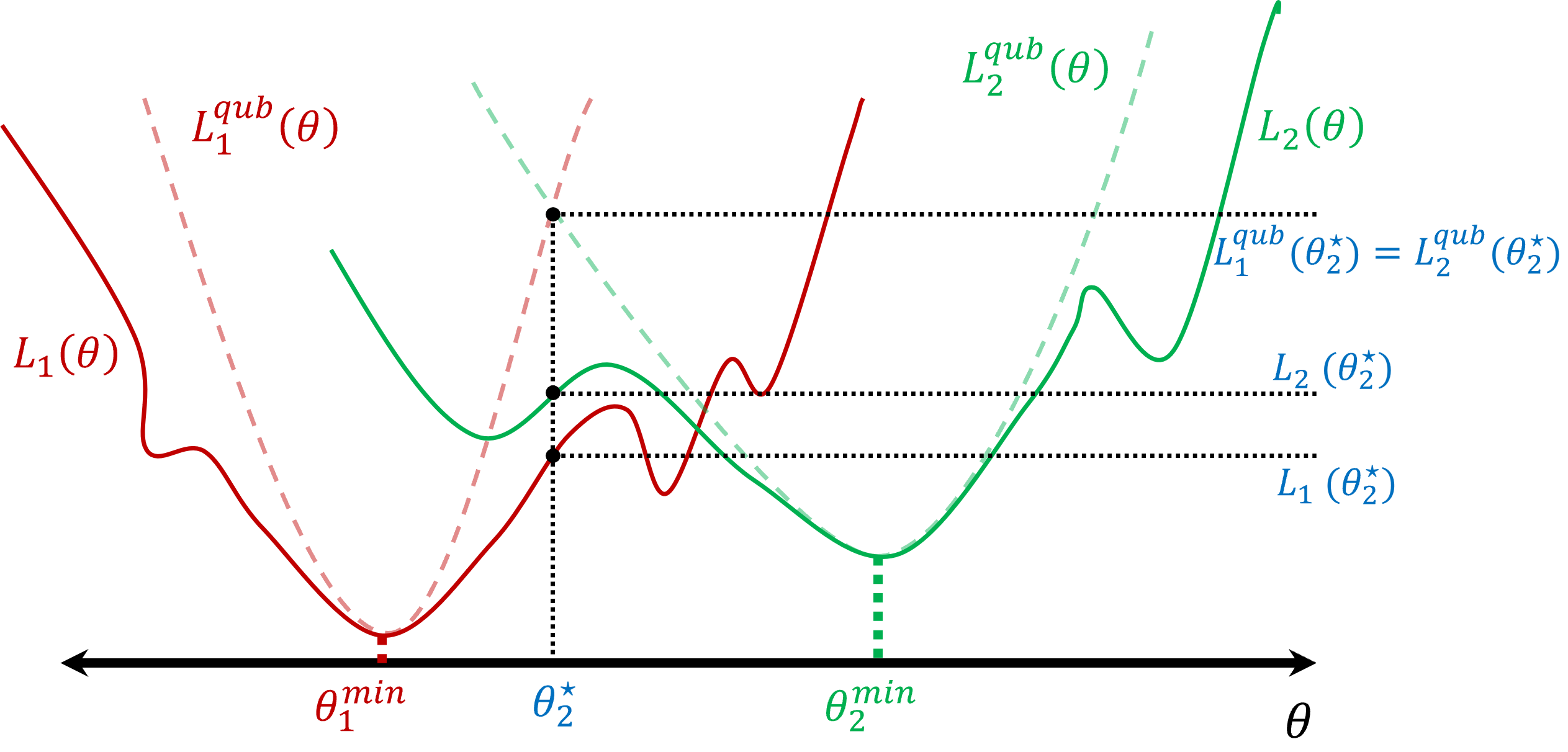}
        \caption{\label{fig:qub_param}\textbf{Parameter domain:} $\vec{\theta}_2^{\star}$
        depicts a PMF-CL solution, with its corresponding QUB and true loss
        values.}
    \end{subfigure}
    \caption{\label{fig:qub_vis}Illustration of the Pareto Fronts (PF) and a
    1-D visualization in parameter domain for a pair of generic task-loss
    functions and their QUB's.}
\end{figure*}

\subsection{Forgetting Analysis for QUB Loss} \label{sec:qub-forgetting}
For the QUB loss functions, we can also derive an upper-bound on the degree of
forgetting.
This is formally stated in \thmref{qub_forgetting} below and proven in
\appref{forgetting-qub}.
\begin{theorem}[Forgetting for QUB] \label{thm:qub_forgetting}
    Suppose $\vec{A}_k$ is defined as in \thmref{mfcl_qub}, $\vec{\nu}_k :=
    \sum_{i=1}^k \alpha_i^{(k)} \vec{H}_i \vec{\theta}_t^{\min}$, and the
 PMF-CL solution generated by \algref{mfcl_qub} up to task $k$ is
    $\vec{\theta}_k^{\star} = \vec{A}_k^{\dagger} \vec{\nu}_k$.
    Then, under \asmref{qub}, the forgetting for task $t < k$ after
    learning task $k$ is bounded as:
    \begin{equation}
        F_{t}^{(k)} \leq
            \frac{1}{2} (\vec{A}_k^{\dagger} \vec{\nu}_k - \vec{\theta}_t^{\min})^{\top}
            \vec{H}_t (\vec{A}_k^{\dagger} \vec{\nu}_k - \vec{\theta}_t^{\min}).
        \label{eq:forgetting-bound-qub}
    \end{equation}
\end{theorem}

Note that, unlike the linear and basis function regression, for QUB loss
functions, forgetting is not always Pareto optimal.
However, \thmref{qub_forgetting} shows that the PMF-CL solution
has a worst-case value for forgetting, as visualized in \figref{qub_pf}.
We depict the PF of the true loss functions and the PF of the respective QUB's.
An ideal algorithm would converge to the true PF, while PMF-CL can converge to any point within the hatched area.
In the parameter domain, \figref{qub_param} illustrates a 1-D case of
two arbitrary loss functions and their QUB's, along with an example PMF-CL
solution.

The core idea and derivations of PMF-CL can also be
extended to the \emph{\ul{F}ederated \ul{C}ontinual \ul{L}earning} (FCL) setting in a straightforward
manner as we prove in \appref{fed-pmfcl}.

\subsection{Multi-class classification as a QUB Problem} \label{sec:multi-class-qub}
In this section, we show that linear multi-class
classification is a QUB problem, and hence admits a static-memory CL algorithm
with bounded forgetting.
The task \emph{cross-entropy} loss for this problem is:
\begin{equation}
    L_t(\vec{\Theta}) = - \frac{1}{n_t} \sum_{i=1}^{n_t} \sum_{k=1}^{K} y_{i,k}^t
        \log\left( \frac{\exp(\vec{\theta}_k^{\top} \vec{x}_i^t)}{
            \sum_{j=1}^K \exp(\vec{\theta}_j^{\top} \vec{x}_i^t)
        } \right)
    \label{eq:task-loss-multiclass}
\end{equation}
where $(y_{i,k}^t)_k$ is a one-hot encoded vector indicating the class of the $\nth{i}$ training sample.
We denote the set of learnable parameters as $\vec{\Theta} := [\vec{\theta}_1,
\dots, \vec{\theta}_K] \in \Real^{d \times K}$, where $\vec{\theta}_k$ are the
weights for class $k$.
Using \lemref{bohning-multiclass} due to B\"ohning
\cite{Bohning:92:AISM:LogReg}, enables us to upper bound the loss function as
(cf. \appref{multiclass-logistic-regression} for the full derivation):
\begin{equation}
    L_t(\vec{\Theta}) \leq L_t(\vec{\Theta}_t^{\min})
    \nonumber \\
    + \frac{1}{2 n_t} \tr{
        (\vec{\Theta} - \vec{\Theta}_t^{\min})^{\top}
        (\vec{X}_t^{\top} \vec{X}_t)
        (\vec{\Theta} - \vec{\Theta}_t^{\min}) \bar{\vec{V}}
        },
        \label{eq:multiclass-qub}
\end{equation}
where $\bar{\vec{V}} = \frac{1}{2} \left[ \vec{I}_K - \frac{1}{K} \vec{1}
\vec{1}^{\top} \right]$, and $\vec{1} \in \Real^{K}$ is an all-$1$ vector. %
\change{The bound in \eqref{multiclass-qub} is a uniform global curvature bound
and can be conservative for the task-specific softmax loss.
Consequently, the resulting PMF-CL update is Pareto-optimal for these quadratic
surrogates, while the associated true-loss forgetting guarantee may be
correspondingly loose.
}
Then, the iterative update of PMF-CL for multi-class classification is summarized in the following corollary of \thmref{mfcl_qub}.

\begin{corollary}[Iterative PMF-CL for logistic loss] \label{cor:mfcl_multi_class}
Using the \change{QUB-surrogate-preserving information} 
of the task $t$ as $\mathcal{I}_t =
\{ \mathcal{S}(\vec{V}_t^e), \vec{\Sigma}_t^e, \vec{\Theta}_t^{\min} \}$, under 
the Givens rotation storage model $\mathcal{S}(\cdot)$, and
$\vec{\alpha}^{(k)}$ as defined in \eqref{preference_global}, the incremental
PMF-CL update for multi-class logistic regression is:
\begin{equation}
    \vec{A}_k \Delta \vec{\Theta}_k = \frac{\alpha_k^{(k)}}{n_k}
        \vec{V}_t^e {\vec{\Sigma}_t^e}^2 {\vec{V}_t^e}^{\top} (\vec{\Theta}_t^{\min} - \vec{\Theta}_{k-1}^{\star}),
\end{equation}
with 
$\vec{A}_0 = \vec{0}$, $\vec{A}_k =
\frac{N_{k-1}}{N_k} \vec{A}_{k-1} + \frac{\alpha_k^{(k)}}{n_k} \vec{V}_k^e
{\vec{\Sigma}_k^e}^2 {\vec{V}_k^e}^{\top}$;
and $\Delta \vec{\Theta}_k = \vec{\Theta}_k^{\star} -
\vec{\Theta}_{k-1}^{\star}$.
\end{corollary}

Memory for the multi-class PMF-CL is static and bounded as $\bigO{Kd + d +
d(d+3)/2}$ to store $\vec{V}_t^e$, $\vec{\Sigma}_t^e$, and the task minimizer $\vec{\Theta}_t^{\min}$.
The resulting forgetting bound is deferred to \appref{multiclass-logistic-regression}.
We validate our findings numerically on the Sequential-MNIST dataset in
\appref{experiments}.


\section{Conclusion}
In this paper, we proposed a solution to the problem of continual
learning with conflicting tasks through the lens of multi-task learning, whereby the learner remembers loss surface representations of older tasks, and jointly optimizes them when learning new tasks.
\change{%
Importantly, we addressed the empirical optimization problem of continual learning.
For settings described in this work, PMF-CL retrieves the same solution as the
corresponding offline joint-training benchmarks on either the true loss (for
exact quadratic) or the loss surrogate (for QUB), and the generalization results
follow from that benchmark.
}
\change{%
Our approach is generally applicable to loss functions admitting a
loss-preserving finite representation and tractable incremental updates.
}
We also derived explicit iterative algorithms for PMF-CL linear and basis-function regression, and for general loss functions with quadratic upper bounds (QUB), which subsumes linear classification problems.
PMF-CL enjoys constant memory costs with respect to the number of tasks.
Future directions of this research include deriving PMF-CL algorithms for more complex models (e.g., standard deep neural networks \emph{without} the QUB property) and population-level generation bounds.
 \label{sec:conclusion}



\subsection*{AI use statement}


\change{In this work, we used generative AI tools to:
generate synthetic data sets,
provide critical ingredients for proving mathematical claims,
and implement methods.
We have not used generative AI tools to:
help develop theoretical models or conceptual frameworks,
formulate mathematical claims,
propose or refine hypotheses,
design or provide feedback on research methodology or experiments,
assist in the writing of proofs,
clean and reformat dataset,
support qualitative and thematic data analysis,
or interpret results;
and:
assisting with translation
is not applicable to this work.
Additionally, we used generative AI tools to create or modify scientific
figures or images, and create or edit software code.
We have reviewed all AI-assisted work. 

We used Codex with GPT-5.6-Sol for writing code on controlled
experiments on linear regression.
The synthetic dataset for linear regression was created via Python code written
by Codex.
We also used ChatGPT for brainstorming the formal proof of minimality of MSI in
the linear regression setup (Appendix B.1.2).
We have verified all AI-generated results and take responsibility for the final content of this work,
including text, claims or artifacts produced with the aid of generative AI.
}





\subsection*{Reproducibility statement}


\change{%
Since our primary contribution is at a theoretical level, the complete proofs
for all the theoretical results are presented in the Appendices
\ref{app:msi}-\ref{app:logistic-reg}.
Our source code is made available as anonymized supplementary material
associated with this submission.
All relevant parameters and experiments are described in the numerical
experiments for linear and logistic regression in \appref{experiments}.
}




\bibliographystyle{iclr_style_files/iclr2027_conference}
\bibliography{bib/refs,bib/foundation}

\newpage
\appendix
\onecolumn

\section{Extended Literature Review}
\label{app:lit_review}

\textbf{Foundations of Continual Learning.} 
The field of CL, or lifelong learning, emerged from early research in robotic agents and task rehearsal \cite{Thrun:95:Springer:RobotLL, Thrun:95:NIPS:Nth, Silver:02:CSCSI:TaskRehearsal}. It is deeply rooted in cognitive science, as the ability to consolidate new information without overwriting existing knowledge is a hallmark of human intelligence \cite{French:99:TICS:CatastrophicForgetting}. In modern applications, maintaining this stability is critical for large-scale models in production, particularly in the biomedical sector \cite{Ray:24:BBInf:LLMHealth, Bell:25:preprint:CL_FM}. 

\textbf{Empirical Methodologies.}
The literature distinguishes three primary algorithmic families to mitigate forgetting \cite{Wang:24:TPAMI:CLSurvey}:
\begin{itemize}
    \item \textbf{Replay-buffer Methods:} \citet{Lopez:17:NIPS:GEM, Chaudhry:18:ICLR:AGEM, Aljundi:19:NIPS:MIR, Buzzega:20:NIPS:DER} utilize a fixed memory buffer to store representative samples or generative features. Despite their success, these methods face challenges regarding memory sufficiency and lack a formal geometric intuition \cite{Verwimp:21:ICCV:Rehearsal}.
    \item \textbf{Regularization Methods:} \citet{Kirkpatrick:17:PNAS:EWC, Zenke:17:ICML:SI, Aljundi:18:ECCV:MAS} introduce penalty terms that constrain model updates. By identifying parameters critical to previous tasks—often via the Fisher Information Matrix—they ensure the model stays within low-loss regions for past objectives.
    \item \textbf{Gradient-Projection Methods:} Works such as \citet{Saha:21:ICLR:GPM, Lin:22:ICLR:TRGP, Lin:22:NIPS:CUBER} focus on the geometry of the gradient space, projecting updates onto the Null space of past feature correlations to ensure zero interference between tasks.
\end{itemize}

\textbf{Theoretical Perspectives.}
Understanding CL through loss geometry has gained significant traction. \citet{Liu:19:preprint:transferability} argue that flatter minima promote better task transferability, while \citet{Benzing:20:preprint:unifying} provide a unifying framework for regularization. Recent work by \citet{Evron:22:CLT:chatastrophic} and \citet{Levinstein:25:preprint:optimalrates} has established formal forgetting rates, with \citet{Levinstein:25:preprint:optimalrates} achieving optimal $O(1/k)$ rates via specific regularization schedules. Furthermore, \citet{Wan:25:ICML:replayforgetting} and \citet{Lin:23:ICML:forgetting} have analyzed how overparameterization and task signal alignment influence whether forgetting is benign.

\textbf{Theoretical Limits and PMF-CL.}
While \citet{Knoblauch:20:ICML:NPHard} established that perfect avoidance of forgetting is NP-hard, PMF-CL navigates this by reformulating CL as a multi-task learning (MTL) problem. Unlike existing methods that do not guarantee Pareto-optimality, PMF-CL uses a quadratic upper-bound framework to compute solutions that are consistent with joint optimization in hindsight, operating within a polynomial time complexity.

\section{Analysis of MSI and Iterative Updates} \label{app:msi}
\subsection{Linear Regression}

\subsubsection{Iterative PMF-CL for Linear Regression (\thmref{mfcl_linear})}
\label{app:mfcl_linear}

\begin{theorem}[Iterative PMF-CL for Linear Regression] \label{thm:mfcl_linear_app}
    Storing the MSI of the task loss as $\mathcal{I}_t^{\mathrm{lin}} :=
    (\vec{V}_t^e, \vec{\Sigma}_t^e, \vec{\tilde{y}}_t)$ from \thmref{msi_linear},
    the PMF-CL update at the $\nth{k}$ task can be written as the solution to:
    \begin{equation}
        \forall~k\in \{2, \dots, T\}\qquad \vec{A}_k \Delta \vec{\theta}_k =
            \frac{\alpha_k^{(k)}}{n_k} \vec{V}_k^e {\vec{\Sigma}_k^e}^2
            (\vec{\tilde{y}}_k - {\vec{V}_k^e}^{\top}
            \vec{\theta}_{k-1}^{\star})
    \end{equation}
    where $\vec{A}_k = \frac{N_{k-1}}{N_k} \vec{A}_{k-1} +
    \frac{\alpha_k^{(k)}}{n_k} \vec{V}_k^e {\vec{\Sigma}_k^e}^2 {\vec{V}_k^e}^{\top}$, $N_k := \sum_{i=1}^k n_i$, and $\Delta \vec{\theta}_k = \vec{\theta}_k^{\star} -
    \vec{\theta}_{k-1}^{\star}$.
\end{theorem}

\begin{proof}
    Defining $\vec{M}_t := \frac{1}{n_t} \vec{V}_t^e {\vec{\Sigma}_t^e}^2 {\vec{V}_t^e}^{\top}$
    and $\vec{\nu}_t := \frac{1}{n_t} \vec{V}_t^e {\vec{\Sigma}_t^e}^2 \vec{\tilde{y}}_t$,
    \Eqref{theta-p-star} upto the $\nth{k}$ task can be simplified as:
    \begin{equation}
        \left( \sum_{t=1}^k \alpha_t^{(k)} \vec{M}_t \right) \vec{\theta}_{k}^{\star}
        = \sum_{t=1}^k \alpha_t^{(k)} \vec{\nu}_t
        \label{eq:theta-k-star}
    \end{equation}
    Then, denoting the $\nth{k}$ solution in terms of the $\nth{(k-1)}$ solution as
    an iterative update, i.e., $\vec{\theta}_k^{\star} = \vec{\theta}_{k-1}^{\star}
    + \Delta \vec{\theta}_k$, we can write:
    \begin{equation}
        \left(\sum_{t=1}^{k-1} \alpha_t^{(k)} \vec{M}_t + \alpha_k^{(k)} \vec{M}_k
        \right) (\vec{\theta}_{k-1}^{\star} + \Delta \vec{\theta}_k)
        = \sum_{t=1}^{k-1} \alpha_t^{(k)} \vec{\nu}_t + \alpha_k^{(k)} \vec{\nu}_k
        \label{eq:linreg-iterative-update-pre}
    \end{equation}
    Note that from the definition of $\alpha_t^{(k)}$ in \eqref{preference_global},
    we have the property $\alpha_t^{(k)} = (N_{k-1} / N_k)\alpha_t^{(k-1)}$ for all
    $t \in [k]$.
    Upon rearranging \eqref{linreg-iterative-update-pre} and recognizing that $\vec{\theta}_{k-1}^{\star}$ satisfies
    $\left(\sum_{t=1}^{k-1} \alpha_t^{(k-1)} \vec{M}_t\right) \vec{\theta}_{k-1}^{\star} =
    \sum_{t=1}^{k-1} \alpha_t^{(k-1)} \vec{\nu}_t$, we get:
    \begin{equation}
        \alpha_k^{(k)} \vec{M}_k \vec{\theta}_{k-1}^{\star}
        + \left( \sum_{t=1}^{k} \alpha_t^{(k)} \vec{M}_t \right) \Delta \vec{\theta}_k
        = \alpha_k^{(k)} \vec{\nu}_k.
        \nonumber
    \end{equation}
    Let us define the accumulated information matrix as $\vec{A}_{k} :=
    \sum_{t=1}^{k} \alpha_t^{(k)} \vec{M}_t$, and the effective target as $\vec{z}_k
    := \alpha_k^{(k)} (\vec{\nu}_k - \vec{M}_k \vec{\theta}_{k-1}^{\star})$.
    Then we can rewrite the above equation as:
    \begin{equation}
        \vec{A}_k \Delta \vec{\theta}_k = \vec{z}_k
        \label{eq:iterative-update}
    \end{equation}
    Finally, we need to find a recursive update for $\vec{A}_k$, which is obtained
    as follows:
    \begin{align}
        \vec{A}_k &= \sum_{t=1}^{k} \alpha_t^{(k)} \vec{M}_t \nonumber \\
            &= \sum_{t=1}^{k-1} \alpha_t^{(k)} \vec{M}_t + \alpha_k^{(k)} \vec{M}_k \nonumber \\
            &= \frac{N_{k-1}}{N_k} \sum_{t=1}^{k-1}
                \alpha_t^{(k-1)} \vec{M}_t +
                \alpha_k^{(k)} \vec{M}_k \nonumber \\
            &= \frac{N_{k-1}}{N_k} \vec{A}_{k-1} + \alpha_k^{(k)} \vec{M}_k 
            \label{eq:Ak-recur}
    \end{align}
    From \eqref{Ak-recur}, \eqref{iterative-update}, and the definiton of
    $\vec{\theta}_k^{\star}$, we get the iterative update in the theorem.
    The iterative method is formally presented in \algref{linreg-cl}.
\end{proof}

\subsubsection{MSI for Linear Regression (\thmref{msi_linear})} \label{app:msi_linear}
\begin{theorem}[MSI for Linear Regression] \label{thm:msi_linear_app}
    Using the eSVD of $\vec{X}_t$, we can write the loss
    function for linear regression as:
    \begin{equation}
        L_t(\vec{\theta}) = \frac{1}{2n_t} \norm{{\vec{V}_t^e}^{\top}
            \vec{\theta} - \tilde{\vec{y}}_t}_{\vec{\Sigma}_t^e}^2
        + \min_{\vec{\theta}' \in \mathcal{W}} L_t(\vec{\theta}')
    \end{equation}
    where we use the notation $\norm{\vec{z}}_{\vec{A}}^2 :=
    \norm{\vec{A} \vec{z}}_2^2$, $\vec{\tilde{y}} = {\vec{\Sigma}_t^e}^{-1}
    \vec{U}_t^e \vec{y}_t$.
    Then, \change{an MSI representing the linear regression loss (upto an
    additive constant) is
    $\mathcal{I}_t^{\mathrm{lin}} := (\mathcal{S}(\vec{V}_t^e), \vec{\Sigma_t^e},
    \vec{\tilde{y}}_t)$, where $\mathcal{S}(\cdot)$ is a compact storage
    model for the orthonormal singular-vector matrix $\vec{V}_t^e$ using Givens
    rotations \cite{Shepard:15:JPCA:Givens}.
    The memory size for this representation is
    $\mathrm{size}(\mathcal{I}_t^{\mathrm{lin}}) = d r_t - \frac{r_t (r_t +
    1)}{2}$ which is provably \emph{minimal}.}
\end{theorem}
\begin{proof}
\change{Under the \emph{full} SVD of $\vec{X}_t$, we can rewrite the task loss in
\eqref{linreg-task-loss} as:}
\begin{align}
    L_t(\vec{\theta})
        &= \frac{1}{2 n_t} \norm{\vec{U}_t \vec{\Sigma}_t {\vec{V}_t}^{\top}
            \vec{\theta} - \vec{y}_t}^2
    \nonumber \\
        &= \frac{1}{2 n_t} \norm{\vec{\Sigma}_t {\vec{V}_t}^{\top} \vec{\theta}
            - \vec{U}_t^{\top} \vec{y}_t}^2 
    \nonumber \\
        &= \frac{1}{2 n_t} \norm{\begin{bmatrix}
                \vec{\Sigma}_t^e & \vec{0} \\
                \vec{0} & \vec{0}
            \end{bmatrix} \vec{V}_t^{\top} \vec{\theta} -
            \begin{bmatrix}
                {\vec{U}_t^e}^{\top} \\
                {\vec{U}_t^{\perp}}^{\top}
            \end{bmatrix} \vec{y}_t}^2
    \label{eq:linreg-task-loss-mi-1}
\end{align}
where we get \eqref{linreg-task-loss-mi-1} from the relationship between SVD and
eSVD.
The matrix $\vec{U}_t^{\perp} \in \Real^{n_t \times (n_t - r_t)}$ corresponds to
the last $(n_t - r_t)$ columns of $\vec{U}_t$.
Finally, we can write:
\begin{equation}
    L_t(\vec{\theta})
        = \frac{1}{2n_t} \norm{
            \vec{\Sigma}_t^e {\vec{V}_t^e}^{\top} \vec{\theta}
            - {\vec{U}_t^e}^{\top} \vec{y}_t
        }^2 + \frac{1}{2n_t} \norm{
            {\vec{U}_t^{\perp}}^{\top} \vec{y_t}
        }^2 
    \label{eq:linreg-task-loss-mi-2}
\end{equation}
By recognizing that $\vec{\Sigma}_t^e$ is square and positive definite, and
using the notations $\norm{\vec{z}}_{\vec{A}}^2 := \vec{z}^T \vec{A}^2
\vec{z}$ and $\tilde{\vec{y}}_t := (\vec{\Sigma}_t^e)^{-1}{\vec{U}_t^e}^{\top}
\vec{y}_t$.
Also, note that the last term $(0.5/n_t)\norm{{\vec{U}_t^{\perp}}^{\top}
\vec{y}_t}^2$ is the minimum task loss attainable by optimizing $\vec{\theta}$.
Hence, we can rewrite \eqref{linreg-task-loss-mi-2} as:
\begin{equation}
    L_t(\vec{\theta})
    = \frac{1}{2n_t} \norm{{\vec{V}_t^e}^{\top} \vec{\theta} -
            \tilde{\vec{y}}_t}_{\vec{\Sigma}_t^e}^2
    + \min_{\vec{\theta}' \in \mathcal{W}} L_t(\vec{\theta}')
    \label{eq:linreg-task-loss-mi}
\end{equation}
Equation \eqref{linreg-task-loss-mi} is a compact representation of the task
loss $L_t(\cdot)$.
In order to \emph{remember} the task \change{(modulo an additive constant)} for
continual learning, the continual learner would need to store the tuple:
\begin{equation}
    \mathcal{I}_t^{\mathrm{lin}} :=(
        \vec{V}_t^e, \diag{\vec{\Sigma}_t^e}, \vec{\tilde{y}}_t).
    \label{eq:linreg-task-mi}
\end{equation}

\change{%
Note that in terms of storage $\mathcal{I}_t$ in \eqref{linreg-task-mi} is not
yet optimized, since the matrix $\vec{V}_t^e$ is orthonormal, i.e., the $dr_t$
entries of $\vec{V}_t^e$ are not independent.

\begin{list}{\labelitemi}{\leftmargin=1.5em \itemindent=-0.0em \itemsep=-.2em}
\item {\bf The Stiefel-manifold dimension}
The matrix $\vec{V}_t^e$ belongs to the real Stiefel manifold
\begin{equation}
    \mathrm{St}(d,r) := \{\vec{V}_t^e \in \Real^{d\times r}: \vec{V}^\top
    \vec{V} = \vec{I}_r\}.
    \label{eq:stiefel}
\end{equation}
The symmetric matrix equation $\vec{V}^\top \vec{V} = \vec{I}_r$ imposes
$\frac{r(r+1)}{2}$ independent scalar constraints: $r$ unit-norm constraints and
$r(r-1)/2$ pairwise-orthogonality constraints.
Therefore,
\begin{equation}
    \dim\mathrm{St}(d,r)=dr-\frac{r(r+1)}{2}.
    \label{eq:stiefel-dimension}
\end{equation}
This standard dimension formula and the geometry behind it are discussed
in \cite{Edelman:98:MAA:Ortho,Absil:08:book:MatrixManifolds}.
Next, we show that \emph{Givens rotations} \cite{Shepard:15:JPCA:Givens}
provides a compact representation that matches the dimensions in
\eqref{stiefel-dimension}.
Finally, we prove that $\mathcal{I}_t$ in \eqref{linreg-task-mi} is indeed
minimal under the storage model represented by Givens rotations.

\item {\bf Compact representation via Givens rotations}
The Stiefel dimension is not merely an abstract parameter count.
At least locally and generically, an orthonormal matrix $\vec{V}^e \in
\mathrm{St}(d,r)$ (dropping the $t$ subscript for convenience) can be encoded
using exactly $\dim\mathrm{St}(d,r)$ real scalars.
Let $\vec{E} = \begin{bmatrix} \vec{I}_r\\ 0 \end{bmatrix} \in \Real^{d\times r}$.
A Givens rotation $\vec{G}_{ij}(\phi)\in O(d)$ is an orthonormal matrix that
acts as a planar rotation through angle $\phi$ in the coordinate plane spanned
by the $\nth{i}$ and $\nth{j}$ standard basis vectors, and acts as the identity
on all other coordinates.

Next, we can apply a QR-style elimination to $\vec{V}^e$.
For column $i$, use Givens rotations to eliminate the entries in rows
$d, (d-1),\ldots, (i+1)$.
This requires $(d-i)$ rotations for column $i$.
Repeating this operation for $i = 1, \ldots, r$ requires
\begin{equation}
    \sum_{i=1}^{r}(d-i) = dr - \frac{r(r+1)}{2} = \dim\mathrm{St}(d,r)
    \label{eq:givens-count}
\end{equation}
rotation angles.

After elimination, the leading $r \times r$ block is both upper triangular
and orthogonal.
It must therefore be a diagonal sign matrix $\vec{D} = \Diag{s_1,\ldots,s_r}$,
for $s_i \in \{-1,+1\}$.
Reversing the elimination gives a representation of the form
\begin{equation}
    \vec{V}^e = \vec{G}_1(\phi_1)^\top
        \vec{G}_2(\phi_2)^\top
        \cdots
        \vec{G}_m(\phi_m)^\top \vec{E} \vec{D},
        \qquad
    m=dr-\frac{r(r+1)}{2},
    \label{eq:givens-representation}
\end{equation}
where the order and coordinate planes of the rotations are fixed in advance.
Thus, only the $m$ real angles and, if desired, $r$ discrete sign bits need to
be stored.
The sign bits do not contribute to the manifold dimension.
Moreover, for the loss representation in \eqref{linreg-task-loss-mi}, the
simultaneous transformation
\begin{equation*}
    \vec{v}_i^e \longmapsto - \vec{v}_i^e,
    \qquad
    \vec{\tilde{y}}_i \longmapsto -\vec{\tilde{y}}_i
    \label{eq:sign-invariance}
\end{equation*}
leaves the loss unchanged.
Consequently, column-sign ambiguity can be removed by a canonical convention
when the goal is to store the loss rather than a particular numerical SVD.

\begin{example}
    Say $d=4$ and $r=2$, the displayed matrix $\vec{V}^e$ has eight entries, but
    $\dim\mathrm{St}(4,2) = 5$.
    It can therefore be represented generically by five fixed-order Givens
    angles, together with only discrete sign conventions.
\end{example}

\item {\bf Minimality of \eqref{linreg-task-mi} under
\eqref{givens-representation}}
We proceed by proving that an alternative representation via the
pair$(\vec{A}$, $\vec{b})$ defined by:
\begin{equation}
    \vec{A} := \nabla^2 L(\vec{\theta}),\qquad
    \vec{b} := \vec{A}\vec{\theta}-\nabla L(\vec{\theta}),
    \label{eq:derivative-recovery}
\end{equation}
fully determines the non-constant part of the loss in
\eqref{linreg-task-loss-mi}, i.e.:
\begin{equation}
    L_t(\vec{\theta}) \equiv \frac{1}{2} \vec{\theta}^\top \vec{A}_t\vec{\theta} -
    \vec{b}_t^\top \vec{\theta},
    \label{eq:loss-canonical}
\end{equation}
and is \emph{minimal}.
Then, we show that both the representations, \eqref{linreg-task-mi} with
the compact \eqref{givens-representation} and \eqref{derivative-recovery} have
identical memory footprint.

In the following proposition, we show that the canonical representation
$(\vec{A}, \vec{b})$ \emph{uniquely} determines the loss function in
\eqref{loss-canonical} upto a constant.
\begin{proposition}[Uniqueness of the canonical coefficients]
Let
\begin{equation}
    q_{\vec{A},\vec{b}}(\vec{\theta}) = \frac{1}{2} \vec{\theta}^\top
    \vec{A}\vec{\theta} - \vec{b}^\top \vec{\theta} + c
    \quad\text{and}\quad
    q_{\vec{A}',\vec{b}'}(\vec{\theta}) = \frac{1}{2} \vec{\theta}^\top
    \vec{A}'\vec{\theta} - (\vec{b}')^\top \vec{\theta} + c',
\end{equation}
where $\vec{A}$ and $\vec{A}'$ are symmetric.
The two functions differ only by an additive constant for every $\vec{\theta}
\in \Real^d$ if and only if $\vec{A}=\vec{A}'$ and $\vec{b}=\vec{b}'$.
\end{proposition}
\begin{proof}
If the difference is constant, its Hessian is zero, which gives
$\vec{A} - \vec{A}' = \vec{0}$.
Its gradient is then $-(\vec{b} - \vec{b}')$, which must also be zero.
The reverse implication is immediate.
\end{proof}

This proves that $(\vec{A}, \vec{b})$ is a minimal required representation to
remember the loss function for linear regression.
Lastly, the memory size of $(\vec{A}, \vec{b})$ can be determined as follows.
Let $r=\rank{\vec{X}}=\rank{\vec{A}}$, after dropping the $t$ subscript for
convenience.
The set of $d\times d$ symmetric p.s.d matrices of fixed rank $r$ is a smooth
manifold of dimension
\begin{equation}
    \dim\{\vec{A} = \vec{A}^\top \succeq 0: \rank{\vec{A}}=r\}
    = dr - \frac{r(r-1)}{2}.
    \label{eq:psd-dimension}
\end{equation}
One way to see this is to write
\begin{equation*}
    \vec{A} = \vec{Y}\vec{Y}^\top,\qquad
    \vec{Y} \in \Real^{d\times r},\quad \rank{\vec{Y}} = r.
\end{equation*}
The factor $\vec{Y}$ contributes $dr$ parameters, but $\vec{Y}$ and
$\vec{Y}\vec{Q}$ give the same $\vec{A}$ for every orthonormal $\vec{Q} \in
O(\vec{r})$.
Since
\begin{equation*}
    \dim O(r)=\frac{r(r-1)}{2},
\end{equation*}
the quotient has the dimension in \eqref{psd-dimension}; see, for
example, \cite{Massart:20:MAA:QuotientGeometry,Meyer:JMLR:11:PSDMatrices}.

For a least-squares loss,
$$ \vec{b} = \frac{1}{n}\vec{X}^\top \vec{y} \in \mathrm{range}(\vec{X}^\top) =
    \mathrm{range}(\vec{X}^\top \vec{X}) = \mathrm{range}(\vec{A}). $$
After $\vec{A}$ is fixed, $\vec{b}$ therefore contributes $r$ additional degrees
of freedom.
Consequently,
\begin{equation}
    \dim(A,b)
    =
    \left(dr-\frac{r(r-1)}{2}\right)+r
    =
    dr-\frac{r(r-3)}{2}.
    \label{eq:ab-dimension}
\end{equation}

\begin{remark}[Realizability]
Every pair with $\vec{A} \succeq \vec{0}$, $\rank{\vec{A}} = r$, and
$\vec{b}\in\mathrm{range}(\vec{A})$ can be realized by a least-squares problem
with sufficiently many rows.
Thus, the dimension count is a property of the actual rank-$r$ least-squares
loss family, not merely of a larger artificial class of quadratics.
\end{remark}

\end{list}
}
\end{proof}

\subsection{Quadratic Upper Bound (\thmref{mfcl_qub})} \label{app:mfcl_qub}

\begin{theorem}[Iterative PMF-CL for QUB loss] \label{thm:mfcl_qub_app}
    Storing the MSI of the task loss as $\mathcal{I}_t = \{\vec{H}_t,
    \vec{\theta}_t^{\min}\}$ and using $\vec{\alpha}^{(k)}$ from
    \eqref{preference_global}, the iterative PMF-CL update can be written as:
    \begin{equation}
        \vec{A}_k \Delta \vec{\theta}_k =
            \alpha_k^{(k)} \vec{H}_k
            (\vec{\theta}_k^{\min} - \vec{\theta}_{k-1}^{\star}).
    \end{equation}
    where $\vec{A}_0 = 0, \vec{A}_k = \frac{N_{k-1}}{N_k} \vec{A}_{k-1} +
    \alpha_k^{(k)} \vec{H}_k$ and $N_k = \sum_{i=1}^k n_i$.
\end{theorem}

\begin{proof}
    The proof mostly mirrors the proof of \thmref{mfcl_linear}.
    Consider the linear equation for the Pareto-optimal solution by PMF-CL,
    i.e.:
    \begin{equation}
        \left(\sum_{t=1}^k \alpha_t^{(k)} \vec{H}_t \right)
            \vec{\theta}_{k}^{\star} = \sum_{t=1}^k \alpha_t^{(k)}
            \vec{H}_t \vec{\theta}_t^{\min}.
    \end{equation}
    Changing the subscript on $\vec{\theta}_{\vec{\alpha}}^{\star}$ to $k$ for
    denoting the $\nth{k}$ iterate, and defining $\Delta \vec{\theta}_k
    := \vec{\theta}_k^{\star} - \vec{\theta}_{k-1}^{\star}$, we can rewrite:
    \begin{align}
        \left(\sum_{t=1}^{k-1} \alpha_t^{(k)} \vec{H}_t + \alpha_k^{(k)}
            \vec{H}_k \right)
            (\vec{\theta}_{k-1}^{\star} + \Delta \vec{\theta}_k)
        &= \sum_{t=1}^{k-1} \alpha_t^{(k)} \vec{H}_t \vec{\theta}_t^{\min}
            + \alpha_k^{(k)} \vec{H}_k \vec{\theta}_k^{\min} 
        \nonumber \\
        \implies \left(\sum_{t=1}^{k} \alpha_t^{(k)} \vec{H}_t \right)
            \Delta \vec{\theta}_k
            + \alpha_k^{(k)} \vec{H}_k \vec{\theta}_{k-1}^{\star}
        &= \alpha_k^{(k)} \vec{H}_k \vec{\theta}_k^{\min}
        \\
        \implies \vec{A}_k \Delta \vec{\theta}_k = \alpha_k^{(k)} \vec{H}_k (
        \vec{\theta}_k^{\min} - \vec{\theta}_{k-1}^{\star})
    \end{align}
    where $\vec{A}_k := \sum_{t=1}^k \alpha_t \vec{H}_t$.
    Again, similar to the linear regression case, we can find a recursive update
    for $\vec{A}_k$ as:
    \begin{equation}
        \vec{A}_k = \frac{N_{k-1}}{N_k} \vec{A}_{k-1} +
            \alpha_k^{(k)} \vec{H}_k,
    \end{equation}
    which completes the proof.
\end{proof}

\section{Forgetting Analysis} \label{app:forgetting-analysis}
\subsection{Linear and Basis Function Regression (\thmref{linreg_forgetting})}
\label{app:forgetting-linear}

\begin{theorem} \label{thm:linreg_forgetting_app}
    The forgetting of the $\nth{t}$ task after learning $k$ tasks using
    \algref{linreg-cl} is given by:
    \begin{equation}
        F_t(\vec{\theta}) = \frac{1}{2 n_t} \norm{
            {\vec{V}_t^e}^{\top} \vec{A}_k^{\dagger} \vec{\nu}_k - \vec{\tilde{y}}_t
        }^2
    \end{equation}
    where $\vec{A}_k$ and $\vec{\nu}_k$ are defined as:
    \begin{subequations}
        \begin{align}
            \vec{A}_k &:= \sum_{t=1}^k \frac{\alpha_t^{(k)}}{n_t} \vec{V}_t^e {\vec{\Sigma}_t^e}^2
                {\vec{V}_t^e}^{\top} \\
            \vec{\nu}_k &:= \sum_{t=1}^k \alpha_t^{(k)} \vec{V}_t^e {\vec{\Sigma}_t^e}^2
                \vec{\tilde{y}}_t.
        \end{align}
        \label{eq:ak-nuk}
    \end{subequations}
    and $\dagger$ represents the Moore-Penrose pseudo-inverse.
\end{theorem}
\begin{proof}
    From the definition of forgetting in \defref{forgetting} at parameter
    $\vec{\theta}$, we have:
    \begin{equation}
        F_t(\vec{\theta}) = L_t(\vec{\theta}) - \inf_{\vec{\theta}' \in
            \mathcal{W}} L_t(\vec{\theta}')
    \end{equation}
    From \thmref{msi_linear}, in the linear regression setting, we get:
    \begin{equation}
        F_t(\vec{\theta}) = \frac{1}{2 n_t} \norm{
            {\vec{V}_t^e}^{\top} \vec{\theta} - \vec{\tilde{y}}_t
        }_{\vec{\Sigma_t^e}}^2
    \end{equation}
    Also, since \algref{linreg-cl} uses a memory efficient implementation for
    computing $\vec{\theta}_k^{\star}$ in \eqref{theta-p-star}, the optimal
    parameters can be written as:
    \begin{equation}
        \vec{\theta}_k^{\star} = \Vec{A}_k^{\dagger} \vec{\nu}_k
    \end{equation}
    where $\vec{A}_k$ and $\vec{\nu}_k$ are defined as in \eqref{ak-nuk} and
    $\vec{A}^{\dagger}$ indicates the Moore-Penrose pseudo-inverse.
    Then, the degree of forgetting the $\nth{t}$ task after learning $k$ tasks
    is given by:
    \begin{equation}
        F_t(\vec{\theta}_k^{\star}) = \frac{1}{2 n_t} \norm{
            {\vec{V}_t^e}^{\top} \vec{A}_k^{\dagger} \vec{\nu}_k - \vec{\tilde{y}}_t
        }^2
    \end{equation}
    Which completes the proof.
\end{proof}

\subsection{Forgetting for QUB Loss (\thmref{qub_forgetting})}
\label{app:forgetting-qub}

\begin{theorem}[Forgetting for QUB Loss] \label{thm:qub_forgetting_app}
    Suppose $\vec{A}_k$ is as defined in \thmref{mfcl_qub}, $\vec{\nu}_k :=
    \sum_{i=1}^k \alpha_i^{(k)} \vec{H}_i \vec{\theta}_t^{\min}$ and the
    PMF-CL solution $\vec{\theta}_k^{\star} = \vec{A}_k^{\dagger} \vec{\nu}_k$
    generated by \algref{mfcl_qub} upto task $k$.
    Then, under \asmref{qub}, the degree of forgetting for task $t < k$ after
    learning task $k$ can be upper-bounded as:
    \begin{equation}
        F_{t}^{(k)} \leq
            \frac{1}{2} (\vec{A}_k^{\dagger} \vec{\nu}_k - \vec{\theta}_t^{\min})^{\top}
            \vec{H}_t (\vec{A}_k^{\dagger} \vec{\nu}_k - \vec{\theta}_t^{\min}).
    \end{equation}
\end{theorem}

\begin{proof}
    Again, similar to the proof for the regression case,
    \thmref{linreg_forgetting_app}, we start from the definition of forgetting
    $F_t(\vec{\theta})$ from \defref{forgetting}, but this time for the QUB
    function:
    \begin{equation}
        F_t^{\mathrm{qub}}(\vec{\theta}) = L_t^{\mathrm{qub}}(\vec{\theta}) -
            \inf_{\vec{\theta}' \in \mathcal{W}} L_t^{\mathrm{qub}}(\vec{\theta}')
    \end{equation}
    where, from the QUB for task $t$ in \eqref{qub_loss_ub} we get:
    \begin{equation}
        F_t^{\mathrm{qub}}(\vec{\theta}) =
            \frac{1}{2} (\vec{\theta} - \vec{\theta}_t^{\min})^{\top}
            \vec{H}_t (\vec{\theta} - \vec{\theta}_t^{\min})
    \end{equation}
    Then at the PMF-CL solution $\vec{\theta}_k^{\star} = \vec{A}_k^{\dagger}
    \vec{\nu}_k$ after learning task $k$, we get:
    \begin{equation}
        F_t^{\mathrm{qub}}(\vec{\theta}_k^{\star}) =
            \frac{1}{2} (\vec{A}_k^{\dagger} \vec{\nu}_k - \vec{\theta}_t^{\min})^{\top}
            \vec{H}_t (\vec{A}_k^{\dagger} \vec{\nu}_k - \vec{\theta}_t^{\min})
    \end{equation}
    where $\vec{A}_k := \sum_{i=1}^k \alpha_i^{(k)} \vec{H}_i$ and $\vec{\nu}_k
    := \sum_{i=1}^k \alpha_i^{(k)} \vec{H}_i \vec{\theta}_i^{\min}$.

    Since the QUB is an upper bound to the true loss, the forgetting measure is
    also an upper bound to the true loss, since $\min_{\vec{\theta} \in
    \mathcal{W}} L_t^{\mathrm{qub}}(\vec{\theta}) = \min_{\vec{\theta} \in
    \mathcal{W}} L_t^{\mathrm{qub}}(\vec{\theta}) = \vec{\theta}_t^{\min}$ by
    definiton.
    Hence, we can conclude that:
    \begin{equation}
        F_t^{(k)} \leq
            \frac{1}{2} (\vec{A}_k^{\dagger} \vec{\nu}_k - \vec{\theta}_t^{\min})^{\top}
            \vec{H}_t (\vec{A}_k^{\dagger} \vec{\nu}_k - \vec{\theta}_t^{\min}),
    \end{equation}
    which completes the proof.
\end{proof}

\section{Proofs of Other Lemmas} \label{app:ntk-2layer}
\subsection{Basis Functions and 2-Layer NTK Models} \label{app:mfcl_basis}
In this section , we show that the framework developed for linear regression in
\secref{mfcl_linreg} can be extended to a larger class of models that are linear
in their parameters.
Consider models $\vec{f}_{\vec{\theta}}(\cdot)$ that are non-linear in their
inputs but linear in $\vec{\theta}$, i.e., models operating via a \emph{basis
function}:
\begin{equation}
    \vec{f}_{\vec{\theta}}(\vec{x}) = \vec{\phi}(\vec{x})^T \vec{\theta},
    \qquad \vec{\phi}(\cdot) : \mathbb{R}^d \to \mathbb{R}^m
    \label{eq:basis-fn-model}
\end{equation}
where $\phi(\cdot)$ is a non-linear \emph{basis function} map
\cite{Hastie:01:Book:ElemsSL}.
These models cover curve fitting applications and have a long history
\cite{Kolb:84:Book:CurveFitting,%
Motulsky:04:Book:BioFitting,Guest:12:Book:NumericalCF}.
For instance, polynomial regression, log-linear regression, exponential
regression, etc, are all basis function methods.
Under the \emph{neural tangent kernel} (NTK) approximation
\cite{Jacot:18:NIPS:NTK, Arora:19:ICML:NTK}, we can also represent a two-layer
neural network as an approximate linear model.
Consider a two-layer multi-layer perceptron (MLP) neural network (NN) with ReLU activations:
\begin{equation}
    f_{\vec{w}, \vec{V}}(\vec{x}) = \vec{w}^T \sigma(\vec{V} \vec{x}),
    \label{eq:2layer-mlp}
\end{equation}
where $\vec{V} \in \mathbb{R}^{m \times d}$ and $\vec{w} \in \mathbb{R}^m$.
We follow \cite{Arora:19:ICML:NTK, Ju:21:ICML:NTKGen} for obtaining the NTK
approximation.
Here, we state the result as a lemma, and delegate the proof to
\appref{ntk-2layer}.
\begin{lemma}[Two-layer NN as Basis Function under NTK] \label{lem:ntk-2layer}
    Under the NTK regime, for small changes in the weights $\Delta \vec{v}$,
    the change in the output of the two-layer MLP in \eqref{2layer-mlp} can be
    approximated as:
    $
        \Delta f_{\vec{w}, \vec{V}}(\vec{x}) \approx \left[\vec{x} \otimes (\vec{w}
        \odot \mathbb{I}_{\vec{V}\vec{x} \geq \vec{0}})\right]^T \Delta \vec{v},
    $
    where $\otimes$ represents the Kronecker product of two matrices.
    Then, under fixed $\vec{w}$, the two-layer MLP can be represented by
    \eqref{basis-fn-model} for:
    $\phi_{\vec{w}, \vec{V}} : \vec{x} \to \vec{x} \otimes (\vec{w} \odot
        \mathbb{I}_{\vec{V} \vec{x} \geq \vec{0}}).
    $
\end{lemma}
\begin{proof}
Consider a 2-layer network\footnote{Here we consider the scalar-output case for
simplicity, but our method is also extendable to the vector-output case.}
$f(\vec{x}) : \mathbb{R}^d \to \mathbb{R}$
parametrized by the weights $\vec{w} \in \mathbb{R}^k$ and $\vec{V} \in
\mathbb{R}^{k \times d}$, given by:
\begin{equation}
    f(\vec{x}) = \vec{w}^T \vec{\sigma}(\vec{V} \vec{x})
\end{equation}
where $\vec{\sigma}(\cdot) : \mathbb{R}^k \to \mathbb{R}^k$ is an element-wise
ReLU function, i.e.:
\begin{equation}
    \vec{\sigma}(\vec{z}) = \mathbb{I}_{\vec{z} \geq \vec{0}} \odot \vec{z},
\end{equation}
$\mathbb{I}_{\vec{z} \geq \vec{0}}$ is an element-wise indicator function,
and $\odot$ is the element-wise product.
Suppose that $\vec{w}$ and $\vec{V}$ are initialized randomly and $\vec{w}$ is
frozen, and $\vec{V}$ is updated.
Then, the change in the output of the network for a given input $\vec{x}$ can be
represented as:
\begin{align}
    f(\vec{x}) + \Delta f(\vec{x}) = \vec{w}^T \left[
        \mathbb{I}_{(\vec{V} + \Delta \vec{V}) \vec{x} \geq \vec{0}}
            \odot (\vec{V} + \Delta\vec{V}) \vec{x}
        \right]
\end{align}
For small changes in input, say $\vec{z} \to \Delta \vec{z}$, the indicator
function $\mathbb{I}_{(\vec{z} + \Delta \vec{z}) \geq 0} \approx
\mathbb{I}_{\vec{z} \geq 0}$ is approximately constant, so we can approximate
the output as (we use the subscript $\vec{w}, \vec{V}$ to make explicit the
dependence of the output on the original randomly chosen parameters.):
\begin{align}
    \Delta f_{\vec{w}, \vec{V}}(\vec{x}) &\approx \vec{w}^T \left[
        \mathbb{I}_{\vec{V}\vec{x} \geq \vec{0}}
            \odot (\vec{V} + \Delta\vec{V}) \vec{x}
        \right]
        - f(\vec{x})
        \nonumber \\
        &= \vec{w}^T \left[
        \mathbb{I}_{\vec{V}\vec{x} \geq \vec{0}}
            \odot \Delta\vec{V} \vec{x}
        \right]
        \label{eq:deltaf-removef}\\
        &= (\vec{w} \odot \mathbb{I}_{\vec{V}\vec{x} \geq \vec{0}})^T
            \Delta\vec{V} \vec{x}
        \\
        &= \tr{\vec{x} (\vec{w} \odot \mathbb{I}_{\vec{V}\vec{x} \geq
            \vec{0}})^T \Delta\vec{V}}
\end{align}
In \eqref{deltaf-removef} we use the definition of $f(\vec{x})$ to cancel out
the first expanded term in the square bracket.
Clearly this is a linear function of $\Delta \vec{V}$.
Suppose $\vec{v}$ is a vectorized version of $\vec{V}$ obtained by stacking its
columns one below the other.
\end{proof}
The mapping in \lemref{ntk-2layer} can be plugged into \algref{linreg-cl} to obtain an approximate
PMF-CL algorithm for two-layer MLPs in the NTK regime.
For basis function models, one can simply change the eSVD on $\vec{X}_t$ in
\algref{linreg-cl} to an eSVD on the dataset in feature space, i.e.,
$\vec{\phi}(\vec{X}_t)$.
Note that the memory complexity would then scale as $\bigO{m^2}$ rather than
$\bigO{d^2}$, where $m$ is the dimension of the feature space.

\subsection{Proof of \lemref{mtl_global}} \label{app:mtl-global}

\begin{lemma}[MTL for global ERM] \label{lem:mtl_global_app}
    The solution to the MTL problem in \eqref{mtl_lin_scalar} with preferences
    chosen as in \eqref{preference_global} is equivalent to the solution of the
    global optimization problem:
    \begin{equation}
        \vec{\theta}^{\star} = \argmin_{\vec{\theta} \in \mathcal{W}}
        \mathcal{L}(\vec{\theta}; \mathcal{D})
    \end{equation}
    where $\mathcal{D}$ is a concatenation of all datasets $\{ \mathcal{D}_1,
    \dots, \mathcal{D}_T \}$.
\end{lemma}
\begin{proof}
    We proceed by showing that both the MTL problem with the sample-weighted
    preferences and the global problem have the same objective function.
    For the linearly scalarized MTL problem with preferences as in
    \eqref{preference_global}, the objective function is given by:
    \begin{align}
        \sum_{t=1}^T \frac{n_t}{n} \mathcal{L} (\vec{\theta}; \mathcal{D}_t)
        &= \frac{1}{n} \sum_{t=1}^T \mathcal{L}(\vec{\theta}; \mathcal{D}_t)
        \\
        &= \frac{1}{n} \sum_{t=1}^T \sum_{i=1}^{n_t}
            \ell(\vec{\theta}; (\vec{x}_t^{(i)},
            y_t^{(i)}))
        \\
        &= \frac{1}{n} \sum_{(\vec{x}, y) \in \mathcal{D}}
            \ell(\vec{\theta}; (\vec{x}, y))
        \\
        &= \mathcal{L}(\vec{\theta}; \mathcal{D})
    \end{align}
    Thus, both problems minimize the same objective function, and hence have the
    same solution.
\end{proof}

\section{Example QUB Loss: Continual Logistic Regression} \label{app:logistic-reg}

\subsection{Binary logistic regression} \label{app:binaryclass-qub}
The method we derived in \secref{mfcl_qub} subsumes the case of continual
logistic regression, as we will show in this appendix.
First, we will look at the simplified setting of binary logistic regression in this section.
The task loss function for logistic regression using a linear model parametrized
by $\vec{\theta} \in \Real^d$ is given as:
\begin{equation}
    L_t(\vec{\theta}) = - \frac{1}{n_t} \sum_{i=1}^{n_t} \left[
        y_i \log\left( p_i \right) + (1 - y_i) \log\left( 1 - p_i \right)
    \right]
\end{equation}
where $p_i = \frac{1}{1 + \exp(-\vec{x}_i^{\top} \vec{\theta})} \in (0, 1)$ is
the sigmoid function applied to the linear model\footnote{We omit the superscript of $t$ on
$\vec{x}_i^t$ and $y_i^t$ to avoid notational clutter.} $f(\vec{\theta}, \vec{x}_i) :=
\vec{x}_i^{\top} \vec{\theta}$.
Computing the first derivative of the loss function with respect to
$\vec{\theta}_t^{\star}$:
\begin{align}
    \nabla_{\vec{\theta}} L_t(\vec{\theta}) &= -\frac{1}{n_t}
        \sum_{i=1}^{n_t} \left[
        \frac{y_i}{p_i} - \frac{1 - y_i}{1 - p_i}
    \right] \nabla_{\vec{\theta}} p_i
    \nonumber \\
    &= -\frac{1}{n_t}
        \sum_{i=1}^{n_t} \left[
        \frac{y_i}{p_i} - \frac{1 - y_i}{1 - p_i}
    \right] p_i (1- p_i) (-\vec{x}_i)
    \nonumber \\
    &= \frac{1}{n_t}
        \sum_{i=1}^{n_t} (y_i - p_i) \vec{x}_i.
\end{align}
And, the second derivative:
\begin{align}
    \nabla_{\vec{\theta}}^2 L_t(\vec{\theta}) &= \frac{1}{n_t} \sum_{i=1}^{n_t}
        \vec{x}_i \left( p_i (1 - p_i) \vec{x}_i^{\top} \right)
    \nonumber \\
    &= \frac{1}{n_t} \vec{X}_t^{\top} \mathrm{Diag}(\vec{q}_t) \vec{X}_t
    \label{eq:logistic-hessian}
\end{align}
where $\vec{q}_t$ is the vector with the $\nth{i}$ element $p_i (1 - p_i)$, and
$\mathrm{Diag}(\cdot)$ is the corresponding diagonal matrix.

Next, note that the quantity $p_i(1 - p_i)$ is upper bounded by $0.25$.
Thus, the $\mathrm{Diag}(\vec{q}_t)$ in the Hessian in \eqref{logistic-hessian},
is always upper bounded by the matrix $\overline{\vec{Q}}_t := 0.25 \vec{I}$, in
the sense that $\left(\overline{\vec{Q}}_t - \mathrm{Diag}(\vec{q}_t)\right)$ is
always positive semi-definite.
Thus, we can construct the upper-bounding surrogate loss function:
\begin{align}
    L_t^{\text{qub}}(\vec{\theta}) &= L_t(\vec{\theta}_t^{\min}) +
        \frac{1}{8 n_t} (\vec{\theta} - \vec{\theta}_t^{\min})^{\top}
        \vec{X}_t^{\top} \vec{X}_t (\vec{\theta} - \vec{\theta}_t^{\min})
    \nonumber \\
    &= L_t(\vec{\theta}_t^{\min}) + \frac{1}{8 n_t} \norm{
        \vec{X}_t \vec{\theta} - \vec{X}_t \vec{\theta}_t^{\min}
    }_2^2
    \\
    &= L_t(\vec{\theta}_t^{\min}) + \frac{1}{8 n_t} \norm{\vec{V}_t^e
        \vec{\theta} - \tilde{\vec{y}}_t}_{\vec{\Sigma}_t^e}^2
\end{align}
where $\vec{U}_t^e \vec{\Sigma}_t^e {\vec{V}_t^e}^{\top}$ is the SVD of
$\vec{X}_t$, and $\tilde{\vec{y}} := (\vec{\Sigma}_t^e)^{-1}
{\vec{U}_t^e}^{\top} \vec{X}_t \vec{\theta}_t^{\min}$.
which is identical to the MSE loss in linear regression with the substitution
$\vec{y}_t := \vec{X}_t \vec{\theta}_t^{\min}$.
As in the linear regression case, the MSI can be reduced significantly by
storing the right singular matrix and the non-zero singular values of $\vec{X}_t$.
Hence the MSI for task $t$ is $\mathcal{I}_t = (\vec{V}_t^e, \vec{\Sigma}_t^e,
\tilde{\vec{y}}_t)$, where $\tilde{\vec{y}}_t := (\vec{\Sigma}_t^e)^{-1}
{\vec{U}_t^e}^{\top} \vec{X}_t \vec{\theta}_t^{\min}$.

Now, the problem reduces to computing the Pareto-optimal solution of the
surrogate task loss functions.
For each task, first we compute the minimizer of the task loss
$\vec{\theta}_t^{\min}$, then, we can directly apply the updates in
\algref{linreg-cl}, with $\vec{y}_t = \vec{X}_t \vec{\theta}_t^{\min}$ to
compute the Pareto-optimal solution of the surrogate loss functions.
Note that since the surrogate functions are upper bounds to the corresponding
true objectives, the solutions computed after each task will satisfy the
forgetting upper bound in \thmref{qub_forgetting}.

\textbf{Deriving the forgetting upper-bound. }
From \thmref{qub_forgetting} the forgetting upper bound of our Continual
Logistic Regression algorithm can be derived as follows.
Recall that for QUB problems, the learned parameters after task $k$, the
forgetting for a previous task $t < k$ is upper bounded by:
\begin{equation}
    F_t^{(k)} \leq \frac{1}{2} (\vec{A}_k^{\dagger} \vec{\nu}_k -
        \vec{\theta}_t^{\min})^{\top} \vec{H}_t (\vec{A}_k^{\dagger} \vec{\nu}_k
        - \vec{\theta}_t^{\min})
    \label{eq:forgetting_qub_logreg}
\end{equation}
where $\vec{A}_k$ is recursively defined as:
$\vec{A}_0 = 0, 
    \vec{A}_k = \frac{N_{k-1}}{N_k} \vec{A}_{k-1} + \alpha_k^{(k)} \vec{H}_k,
    ~~\forall~ k \in [T].$
and $\vec{\nu}_k := \sum_{i=1}^k \alpha_i^{(k)} \vec{H}_i
\vec{\theta}_i^{\min}$.
For logistic regression, the Hessian upper bound $\vec{H}_t$ is given by,
$\vec{H}_t = \frac{1}{4n_t} \vec{X}_t^{\top} \vec{X}_t$.
Substituting this into \eqref{forgetting_qub_logreg} we obtain:
\begin{equation}
    F_t^{(k)} \leq \frac{1}{8 n_t} \norm{\vec{X}_t \vec{A}_k^{\dagger} \vec{\nu}_k
        - \vec{X}_t \vec{\theta}_t^{\min}}^2,
\end{equation}
where $\vec{\nu}_k := \sum_{i=1}^k \frac{\alpha_i^{(k)}}{4 n_i} \vec{X}_i^{\top} \vec{X}_i
\vec{\theta}_i^{\min}$.

\subsection{Multi-class logistic regression} \label{app:multiclass-logistic-regression}
Note that the loss for multi-class logistic regression with $K$ classes is given by:
\begin{equation}
    L_t(\vec{\theta}) = - \frac{1}{n_t} \sum_{i=1}^{n_t} \sum_{k=1}^{K} y_{i,k}^t
        \log\left( \frac{\exp(\vec{\theta}_k^{\top} \vec{x}_i^t)}{
            \sum_{j=1}^K \exp(\vec{\theta}_j^{\top} \vec{x}_i^t)
        } \right).
    \label{eq:task-loss-multiclass-app}
\end{equation}
Denoting $p_{i,k} := \frac{\exp(\vec{\theta}_k^{\top}
\vec{x}_i^t)}{\sum_{j=1}^K \exp(\vec{\theta}_j^{\top} \vec{x}_i^t)}$, and
$\vec{p}_i = [p_{i,1}, \dots p_{i,K}]^{\top}$ as the output of the softmax
function, the Hessian for \eqref{task-loss-multiclass} is given by:
\begin{equation}
    \nabla_{\vec{\Theta}}^2 L_t = \frac{1}{n_t} \sum_{i=1}^{n_t}
        \left( \Diag{\vec{p}_i} - \vec{p}_i \vec{p}_i^{\top} \right)
        \otimes \left( \vec{x}_i \vec{x}_i^{\top} \right).
    \label{eq:hessian-multiclass}
\end{equation}
We first state the following lemma due to B\"ohning
\citep{Bohning:92:AISM:LogReg} that will be useful to derive the Hessian upper
bound for multi-class logistic regression.
\begin{lemma}[\citet{Bohning:92:AISM:LogReg}] \label{lem:bohning-multiclass}
    For a probability vector $\vec{p} \in \Delta^{K-1}$, the matrix
    $\vec{V}(\vec{p}) := \Diag{\vec{p}} - \vec{p} \vec{p}^{\top}$ is upper bounded by:
    \begin{equation}
        \bar{\vec{V}} := \frac{1}{2} \left[ \vec{I}_K - \frac{1}{K} \vec{1}
            \vec{1}^{\top} \right],
    \end{equation}
    where $\vec{1}$ is the all-ones vector of size $K$.
    In other words $\bar{\vec{V}} - \vec{V}(\vec{p})$ is positive semi-definite.
\end{lemma}
Substituting the bound $\vec{V}(\vec{p}) \preceq \bar{\vec{V}}$ into
\eqref{hessian-multiclass}, we obtain the Hessian upper bound as:
\begin{equation}
    \vec{H}_t := \frac{1}{n_t} \bar{\vec{V}} \otimes \vec{X}_t^{\top} \vec{X}_t
    \label{eq:hessian-ub-multiclass}
\end{equation}

\textbf{Computing the Loss Upper Bound. }
From \eqref{hessian-ub-multiclass}, we can compute the iterative PMF-CL update
for multi-class classification as follows.  First, the upper bound for the
$\nth{t}$ task is given by:
\begin{align}
    L_t^{\mathrm{ub}}(\vec{\Theta})
    &= L_t(\vec{\Theta}_t^{\min}) + \frac{1}{2 n_t} \mathrm{vec}(\vec{\Theta} -
        \vec{\Theta}_t^{\min})^{\top} \left( \bar{\vec{V}} \otimes
        \vec{X}_t^{\top} \vec{X}_t \right) \mathrm{vec}(\vec{\Theta} -
        \vec{\Theta}_t^{\min})
    \\
    &= L_t(\vec{\Theta}_t^{\min}) + \frac{1}{2 n_t} \tr{
        (\vec{\Theta} - \vec{\Theta}_t^{\min})^{\top}
        (\vec{X}_t^{\top} \vec{X}_t)
        (\vec{\Theta} - \vec{\Theta}_t^{\min}) \bar{\vec{V}}
        }
    \\
    &= L_t(\vec{\Theta}_t^{\min}) + \frac{1}{2 n_t} \tr{
        (\vec{\Theta} - \vec{\Theta}_t^{\min})^{\top}
        (\vec{X}_t^{\top} \vec{X}_t)
        (\vec{\Theta} - \vec{\Theta}_t^{\min}) \bar{\vec{V}}
        }
    \label{eq:loss_ub_trace_multiclass}
\end{align}
where $\mathrm{vec}(\cdot)$ is the vectorization operator.
We use the property that $\mathrm{vec}(\vec{A})^{\top} (\vec{B} \otimes \vec{C})
\mathrm{vec}(\vec{A}) = \tr{\vec{A}^{\top} \vec{C} \vec{A} \vec{B}}$ for any
matrices $\vec{A} \in \Real^{d \times K}$, symmetric matrix $\vec{B} \in \Real^{K \times K}$, and gram matrix
$\vec{X}^{\top} \vec{X} =: \vec{C} \in \Real^{d \times d}$.

\textbf{Iterative PMF-CL Update. }
Computing the linear scalarization of the upper bounds, and setting its
derivative with respect to $\vec{\Theta}$ to zero, we get the following linear
equation:
\begin{equation}
    \left(\sum_{t=1}^k \frac{\alpha_t^{(k)}}{n_t} (\vec{X}_t^{\top} \vec{X}_t)
        \vec{\Theta}\right) \bar{\vec{V}}
        = \left(\sum_{t=1}^k \frac{\alpha_t^{(k)}}{n_t} \vec{X}_t^{\top} \vec{X}_t \vec{\Theta}_t^{\min}\right) \bar{\vec{V}}
\end{equation}
\change{The centering matrix $\bar{\vec{V}}$ has a rank of $K-1$ making it singular.
Thus, it cannot be dropped directly.
However, note that the classifier is still defined through the softmax function,
which has a gauge-invariance property, i.e.: $\softmax(\vec{x}^T \vec{\Theta}
\vec{P}) = \softmax(\vec{x}^T \vec{\Theta})$ for any $\vec{\Theta} \in
\mathcal{W}$, and $\vec{P} = 2 \bar{\vec{V}}$ is the centering matrix.
Hence, the centered $\vec{\Theta}$, i.e., $\vec{\Theta}_c = \vec{\Theta}
\vec{P}$, is an equivalent solution to $\vec{\Theta}$ itself:
\begin{equation}
    \left(\sum_{t=1}^k \frac{\alpha_t^{(k)}}{n_t} \vec{X}_t^{\top} \vec{X}_t\right) \vec{\Theta}_c
        = \left(\sum_{t=1}^k \frac{\alpha_t^{(k)}}{n_t} (\vec{X}_t^{\top}
        \vec{X}_t) \vec{\Theta}_t^{\min} \right) \vec{P}
    \label{eq:theta-const-shift}
\end{equation}
Hereafter, we define $\vec{\Theta}_{c,t}^{\min} = \vec{\Theta}_t^{\min}
\vec{P}$ as the centered version of the task-loss minimizer.
}
Then, the iterative solution can be determined by defining the accumulation
matrix $\vec{A}_k$ as:
\begin{equation}
    \vec{A}_k := \sum_{t=1}^k \frac{\alpha_t^{(k)}}{n_t} \vec{X}_t^{\top} \vec{X}_t
        = \frac{N_{k-1}}{N_k} \vec{A}_{k-1} + \frac{\alpha_k^{(k)}}{n_k}
        \vec{V}_k^e {\vec{\Sigma}_k^e}^2 {\vec{V}_k^e}^{\top}
\end{equation}
where we expand the economy-SVD form of $\vec{X}_k$.
Then, the iterative PMF-CL update follows from \eqref{pmfcl_qub_ak_update} with
$\vec{H}_k$ replaced by $\frac{1}{n_k} \vec{X}_k^{\top} \vec{X}_k$, i.e.:
\begin{equation}
    \vec{A}_k \Delta \vec{\Theta}_{c,k} = \frac{\alpha_k^{(k)}}{n_k}
        \vec{V}_k^e {\vec{\Sigma}_k^e}^2 {\vec{V}_k^e}^{\top}
        (\vec{\Theta}_{c,k}^{\min} - \vec{\Theta}_{c,k-1}^{\star})
\end{equation}
which is identical to the iterative update of linear regression with
$\tilde{y}_t$ replaced with ${\vec{V}_t^e}^{\top} \vec{\Theta}_{c,k}^{\min}$.
This iterative update enables us to compute the PMF-CL solution with a
worst-case static memory footprint of $\bigO{d^2}$ to store the singular values
and the right singular matrix.
\change{The full solution would then lie in the set:
\begin{equation}
    \vec{\Theta}_{k}^{\star} \in \{\vec{\Theta}_{c,k}^{\star} + \vec{u} \vec{1}^{\top} :
    \vec{u} \in \Real^d\}
\end{equation}}

\textbf{Forgetting bound. }
From the loss upper bound obtained in \eqref{loss_ub_trace_multiclass}, we can
get the forgetting bound of task $t$ after training upto task $k$ as (using
centered parameters):
\begin{equation}
    F_t^{(k)} \leq \frac{1}{2 n_t} \tr{
        (\vec{\Theta}_c - \vec{\Theta}_{c,t}^{\min})^{\top}
        (\vec{X}_t^{\top} \vec{X}_t)
        (\vec{\Theta}_c - \vec{\Theta}_{c,t}^{\min}) \bar{\vec{V}}
        }.
\end{equation}
Note that $\bar{\vec{V}}$ is a symmetric p.s.d. matrix, hence we can write
$\bar{\vec{V}} = \bar{\vec{V}}^{1/2} \bar{\vec{V}}^{1/2}$.
Hence, we get:
\begin{align}
    F_t^{(k)} &\leq \frac{1}{2 n_t} \tr{
        (\vec{\Theta}_c - \vec{\Theta}_{c,t}^{\min})^{\top}
        (\vec{X}_t^{\top} \vec{X}_t)
        (\vec{\Theta}_c - \vec{\Theta}_{c,t}^{\min}) \bar{\vec{V}}^{1/2}
        \bar{\vec{V}}^{1/2}
    }
    \\
    &= \frac{1}{2 n_t} \tr{
        \bar{\vec{V}}^{1/2} (\vec{\Theta}_c - \vec{\Theta}_{c,t}^{\min})^{\top}
        (\vec{X}_t^{\top} \vec{X}_t)
        (\vec{\Theta}_c - \vec{\Theta}_{c,t}^{\min}) \bar{\vec{V}}^{1/2}
    }
    \\
    &= \frac{1}{2 n_t} \norm{
        \vec{X}_t
        (\vec{\Theta}_c - \vec{\Theta}_{c,t}^{\min}) \bar{\vec{V}}^{1/2}
    }_{F}^2
\end{align}
where the last equality is because $\bar{\vec{V}}$ is symmetric.
The final bound is obtained by replacing $\vec{\Theta}_c$ with the value
computed by PMF-CL, which is $\vec{\Theta}_c^{\star} = \vec{A}_k^{\dagger}
\vec{\nu}_k$, where
$\vec{\nu}_k = \sum_{t=1}^k \frac{\alpha_t^{(k)}}{n_t} (\vec{X}_t^{\top}
\vec{X}_t) \vec{\Theta}_{c,t}^{\min}$.
Then, we obtain the final forgetting bound as:
\begin{equation}
    F_t^{(k)} \leq \frac{1}{2 n_t} \norm{
        \vec{X}_t
        (\vec{A}_k^{\dagger} \vec{\nu}_k - \vec{\Theta}_{c,t}^{\min}) \bar{\vec{V}}^{1/2}
    }_{F}^2
\end{equation}

\section{Extension to a Federated CL setting} \label{app:fed-pmfcl}
Our results in Sections \ref{sec:mfcl}-\ref{sec:mfcl_qub} can be extended to a
\ul{F}ederated \ul{C}ontinual \ul{L}earning (FCL) setting as well.
While various works have studied FCL in different settings
\cite{Yoon:21:ICML, Shenaj:23:CVPR, Wang:24:CVPR, Zhang:23:CVPR,
Usmanova:21:preprint, Dong:22:CVPR, Keshri:24:preprint, Han:23:PMLR},
note that these works are predominantly empirically motivated, and it is
challenging to guarantee their convergence in a Pareto-optimal sense.
A more comprehensive list of related works on FCL can be found in the surveys
\cite{Yang:24:FCL:survey, Wang:24:FCL:survey, Criado:22:FCL:survey}.
In the following, we first explain our FCL setting and then derive a federated
iterative variant of the PMF-CL algorithm.

\subsection{Federated Continual Learning: Setting}
In our federated learning setting, we assume that there are $m$ participating
clients with evolving datasets.
The dataset $\mathcal{D}_{t,i}$ arrives at time-step $t \in [T]$, at client $i$;
and for simplicity, this process is synchronous across clients.
Orchestrated by a central server, the $m$ clients aim to collaboratively learn a
common model $\vec{f}_{\vec{\theta}}(\cdot)$ parametrized by $\vec{\theta}$
without sharing their private datasets.
As in \secref{mfcl} the goal is to compute a Pareto-optimal solution to
$\vec{\theta}_k^{\star}$ at each time-step $k \in [T]$ by saving task-loss-aware
MSI for each task.

\textbf{Preference Vectors. }
Previously, in \eqref{preference_global} we defined the preference vector for
each task.
We appropriately modify it to also include a client index $i$ as follows.
Denoting the number of samples in the dataset $\mathcal{D}_{i,t}$ as $n_{t,i}$,
and the net total samples seen at client $i$ upto time-step $k$ as $N_{k,i} :=
\sum_{t=1}^k n_{t,i}$:
\begin{equation}
    \alpha_{t,i}^{(k)} := \frac{n_{t,i}}{\sum_{t'=1}^k n_{t',i}}.
    \label{eq:fcl-pref-alpha}
\end{equation}
For the federated setting, we also introduce a preference vector for each
client, based on the relative number of samples seen by each client upto
time-step $k$ as:
\begin{equation}
    \beta_{i}^{(k)} = \frac{N_{k,i}}{\sum_{i'=1}^m N_{k,i'}}.
    \label{eq:fcl-pref-beta}
\end{equation}

\begin{algorithm}[t]
    \caption{\label{alg:fed-pmf-cl}Federated PMF-CL for QUB loss functions}
    \begin{algorithmic}[1]
        \STATE \textbf{Input:} Number of clients $m$, number of time-steps $T$.
        \STATE \textbf{Initialize:} $\vec{\theta}_0^{\star}$,
        $\alpha_{0,i}^{(0)}=1$ and $\beta_i^{(0)}=1$ for all $i \in [m]$.
        \STATE \textbf{Initialize:} For all clients $i \in [m]$, $\vec{B}_{0,i}
            = \vec{0}$ and $N_{0,i}=0$.
        \FOR{$k=1$ to $T$}
            \FOR{each client $i=1$ to $m$ \textbf{in parallel}}
                \STATE \textbf{Compute} $\vec{\theta}_{k,i}^{\min}$ and $\vec{H}_{k,i}$
                    based on $\mathcal{D}_{k,i}$.
                \STATE \textbf{Update} $N_{k,i} = N_{k-1,i} + n_{k,i}$.
                \STATE \textbf{Update} $\vec{B}_{k,i} = 
                    \frac{N_{k-1,i}}{N_{k,i}} \vec{B}_{k-1,i} +
                    \alpha_{k,i}^{(k)} \vec{H}_{t,i}$.
                \STATE \textbf{Compute} $\vec{\gamma}_{k,i} = \alpha_{k,i}^{(k)} \vec{H}_{k,i}
                    \left( \vec{\theta}_{k,i}^{\min} - \vec{\theta}_{k-1}^{\star} \right)$.
                \STATE \textbf{Send} $N_{k,i}$, $\vec{B}_{k,i}$ and
                    $\vec{\gamma}_{k,i}$ to the central server.
            \ENDFOR
            \STATE At server, \textbf{compute} preferences $\beta_i^{(k)}$ for
                all $i \in [m]$ based on $\{N_{k,i}\}_{i=1}^m$.
            \STATE \textbf{Aggregate} $\vec{B}_k =
                \sum_{i=1}^m \beta_i^{(k)} \vec{B}_{k,i}$ and $\vec{\gamma}_k =
                \sum_{i=1}^m \beta_i^{(k)} \vec{\gamma}_{k,i}$.
            \STATE Solve for $\Delta \vec{\theta}_k$ in $\vec{B}_k \Delta \vec{\theta}_k =
                \vec{\gamma}_k$.
            \STATE \textbf{Update} $\vec{\theta}_k^{\star} = \vec{\theta}_{k-1}^{\star} +
                \Delta \vec{\theta}_k$.
            \STATE \textbf{Send} $\vec{\theta}_k^{\star}$ to all clients.
        \ENDFOR
        \STATE \textbf{Output:} $\{\vec{\theta}_k^{\star}\}_{k=1}^T$.
    \end{algorithmic}
\end{algorithm}

\subsection{Federated PMF-CL}

As in \secref{mfcl_qub}, we will consider problems with a quadratic upper bound
(QUB) structure.
The Federated PMF-CL algorithm for linear or basis function regression would
follow a very similar derivation, and is omitted here for brevity.
At a time-step $k$ during the continual learning process, the Pareto-optimal
solution to all the tasks seen by all clients so far\footnote{This is simply an
extension of \eqref{mfcl_qub_solution} to the case of datasets distributed
across clients (the space dimension), along with the time dimension.}, would be
given by:
\begin{equation}
    \left(
        \sum_{t=1}^k \sum_{i=1}^m
        \frac{n_{t,i}}{\sum_{t'=1}^k \sum_{i'=1}^m n_{t',i'}} \vec{H}_{t,i}
    \right) \vec{\theta}_k^{\star} = \sum_{t=1}^k \sum_{i=1}^m 
        \frac{n_{t,i}}{\sum_{t'=1}^k \sum_{i'=1}^m n_{t',i'}} \vec{H}_{t,i}
        \vec{\theta}_{t,i}^{\min}
\end{equation}
where $\vec{H}_{t,i}$ is the Hessian upper bound of the loss function
$L_{t,i}(\vec{\theta})$, and $\vec{\theta}_{t,i}^{\min}$ is the optimal solution
corresponding to $\mathcal{D}_{t,i}$.
In terms of the preferences denoted in \eqref{fcl-pref-alpha} and
\eqref{fcl-pref-beta}, the Pareto-optimal solution can be written as:
\begin{equation}
    \left( 
        \sum_{i=1}^m \beta_i^{(k)} \sum_{t=1}^k \alpha_{t,i}^{(k)} \vec{H}_{t,i}
    \right) \vec{\theta}_k^{\star} =
        \sum_{i=1}^m \beta_i^{(k)} \sum_{t=1}^k \alpha_{t,i}^{(k)} \vec{H}_{t,i}
        \vec{\theta}_{t,i}^{\min}
    \label{eq:fcl-qub-solution}
\end{equation}
Next, we proceed to derive an iterative algorithm to compute
$\vec{\theta}_k^{\star}$ at each time-step $k$ based on the previous optimal
model $\vec{\theta}_k^{\star}$.
We can rewrite the above equation as:
\begin{equation}
    \left[
        \sum_{i=1}^m \beta_i^{(k)} \sum_{t=1}^{k-1} \alpha_{t,i}^{(k)} \vec{H}_{t,i}
        + \sum_{i=1}^m \beta_i^{(k)} \alpha_{k,i}^{(k)} \vec{H}_{k,i}
    \right] \vec{\theta}_k^{\star}
        =
        \sum_{i=1}^m \beta_i^{(k)} \sum_{t=1}^{k-1} \alpha_{t,i}^{(k)} \vec{H}_{t,i}
        \vec{\theta}_{t,i}^{\min} + \sum_{i=1}^m \beta_i^{(k)}\alpha_{k,i}^{(k)} \vec{H}_{k,i}
        \vec{\theta}_{k,i}^{\min}
\end{equation}
Next, we introduce the model update $\Delta\vec{\theta}_k =
\vec{\theta}_k^{\star} - \vec{\theta}_{k-1}^{\star}$, and use the recursive
rules for $\alpha_{t,i}^{(k)}$ and $\beta_{i}^{(k)}$:
\begin{align}
    \alpha_{t,i}^{(k)} &= \alpha_{t,i}^{(k-1)}\frac{N_{k-1,i}}{N_{k,i}} \\
    \beta_{i}^{(k)} &= \beta_{i}^{(k-1)} \frac{N_{k,i}}{N_{k-1,i}} .
        \frac{\sum_{j=1}^m N_{k-1,j}}{\sum_{j=1}^m N_{k,j}}.
\end{align}
Also, recognizing that \eqref{fcl-qub-solution} holds for $k-1$ as well, we can
simplify the above as:
\begin{equation}
    \left[
        \sum_{i=1}^m \beta_i^{(k)} \sum_{t=1}^{k} \alpha_{t,i}^{(k)} \vec{H}_{t,i}
    \right] \Delta \vec{\theta}_k = \sum_{i=1}^m \beta_i^{(k)} \alpha_{k,i}^{(k)} \vec{H}_{k,i}
        \left( \vec{\theta}_{k,i}^{\min} - \vec{\theta}_{k-1}^{\star} \right).
    \label{eq:iterative-fcl-update}
\end{equation}
The iterative update in \eqref{iterative-fcl-update} is straightforward to
implement in a federated learning setup.
In \algref{fed-pmf-cl}, we present an implementation of a federated PMF-CL
algorithm that solves \eqref{iterative-fcl-update} in a distributed fashion.

\textbf{A note on memory requirements. }
The federated PMF-CL \algref{fed-pmf-cl} would require the clients to upload the
triple $(N_{k,i}, \vec{B}_{k,i}, \vec{\gamma}_{k,i})$ at each iteration, which
would be a total of $\bigO{d^2}$ floating-point values.
This value is much larger compared to conventional FedAvg
\cite{McMahan:16:CoRR}, but is perhaps the price paid to arrive at
Pareto-optimal solutions.
As in \secref{mfcl_qub}, one can replace the Hessian matrices in $\vec{B}_{k,i}$
with their diagonal approximations.
This would bring down the memory requirement to $\bigO{d}$, at the cost of
approximation.
The download from the server would require $\bigO{d}$ memory, which is identical
to conventional FedAvg.
A systematic study to optimize the communication and memory requirements of
\algref{fed-pmf-cl} is an interesting direction for future work.

\section{Experimental Simulation Results} \label{app:experiments}
To empirically validate the efficacy of our proposed PMF-CL framework, we evaluate its performance against several established continual learning baselines.
Our experiments aim to answer the following questions: (1) How well does PMF-CL mitigate catastrophic forgetting compared to rehearsal and regularization-based methods?
(2) Can PMF-CL successfully converge to the global Pareto-optimal parameter set across multiple tasks?
(3) What are the computational trade-offs regarding memory footprint and execution time?

\subsection{Experimental Setup and Implementation}
\label{subsec:setup}

\textbf{Task Design:} We evaluate the proposed methods on a synthetic continual learning benchmark consisting of 10 sequential tasks.
For each task $k$, the input features $X_k$ are sampled from a Gaussian distribution, while regression targets $y_k$ are generated using ground truth parameters $\theta_k^*$ that undergo a controlled stochastic shift between tasks.
Specifically, for task $k=1$, $\theta_1^*$ is initialized randomly, and for subsequent tasks $k > 1$, the parameters are updated via $\theta_k^* = \theta_{k-1}^* + \epsilon$, where $\epsilon \sim \mathcal{N}(0, \sigma^2 I)$ simulates drifting local minimizers.

\textbf{PMF-CL Implementation:} Our framework utilizes a Pareto-minimal forgetting approach implemented within the \textit{Mammoth} continual learning framework \cite{mammoth}.
The model employs Singular Value Decomposition (SVD) on task-specific data to facilitate efficient parameter updates by storing only the decomposed components ($V_t, \Sigma_t$) rather than full covariance matrices.
These components are used to incrementally update an aggregate curvature matrix $A_t = (N_{t-1}/N_t) A_{t-1} + \alpha^{(t)} V_t \Sigma_t^2 V_t^T$, where $\alpha^{(t)}$ serves as a preference weight to modulate task importance.
To ensure numerical stability, we apply Tikhonov regularization by adding $\lambda I$ (with $\lambda = 10^{-6}$) to the curvature matrix before solving for the Pareto-optimal parameter shift $\Delta\theta$.
The final update $\theta_t = \theta_{t-1} + \Delta\theta$ allows the model to shift toward the global optimum of all observed tasks without requiring access to previous raw datasets.

\textbf{Evaluation Protocol:} All reported metrics and plotted learning curves represent the average performance across 10 independent trials using randomly generated synthetic datasets.
Shaded regions in the plots denote $\pm 1$ standard deviation from the mean.

\textbf{Baselines:} We evaluate the performance of PMF-CL against three standard CL approaches, and a naive SGD baseline which completely forgets previous tasks:
\begin{itemize}
    \item \textbf{SGD (Naive Fine-Tuning):} Sequential training on new tasks using Stochastic Gradient Descent without any memory or regularization.
    This serves as an empirical lower bound, illustrating the maximum extent of catastrophic forgetting.
    \item \textbf{Experience Replay (ER)} \cite{Chaudhry:19:preprint:TinyEM}\textbf{:} A rehearsal strategy that maintains a fixed-size memory buffer of past episodic data, randomly interleaving it with the current task's data to anchor historical knowledge.
    \item \textbf{Dark Experience Replay (DER)} \cite{Buzzega:20:NIPS:DER}\textbf{:} An advanced rehearsal baseline that stores past data alongside their pre-softmax logits, enforcing consistency by penalizing deviations from historical outputs.
    \item \textbf{Synaptic Intelligence (SI)} \cite{Zenke:17:ICML:SI}\textbf{:} A regularization-based approach that computes the importance of individual parameters online, applying a surrogate loss penalty to restrict updates on weights deemed critical for preceding tasks.
\end{itemize}

\textbf{Compute resources. } All experiments were conducted using a single workstation equipped with an NVIDIA H100 GPU.
Numerical operations were implemented using the PyTorch framework \cite{pytorch} in Python.

\subsection{Prediction Accuracy and Parameter Convergence}
\label{subsec:accuracy_convergence}

We first assess the predictive performance and theoretical convergence properties of the models.
Figure \ref{fig:prediction_accuracy} illustrates the Mean Squared Error (MSE) across tasks and the parameter distance to the global optimum.

\begin{figure}[ht]
    \centering
    \begin{subfigure}[b]{0.32\textwidth}
        \centering
        \includegraphics[width=\textwidth]{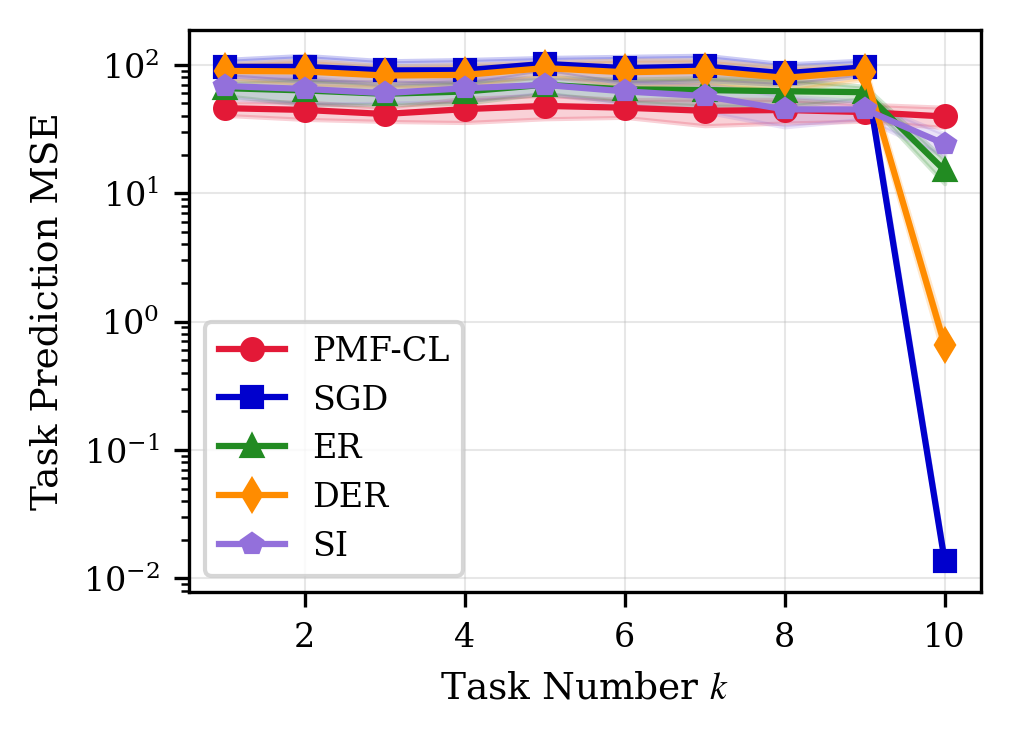}
        \caption{Current Task Accuracy}
        \label{fig:task_pred}
    \end{subfigure}
    \hfill
    \begin{subfigure}[b]{0.32\textwidth}
        \centering
        \includegraphics[width=\textwidth]{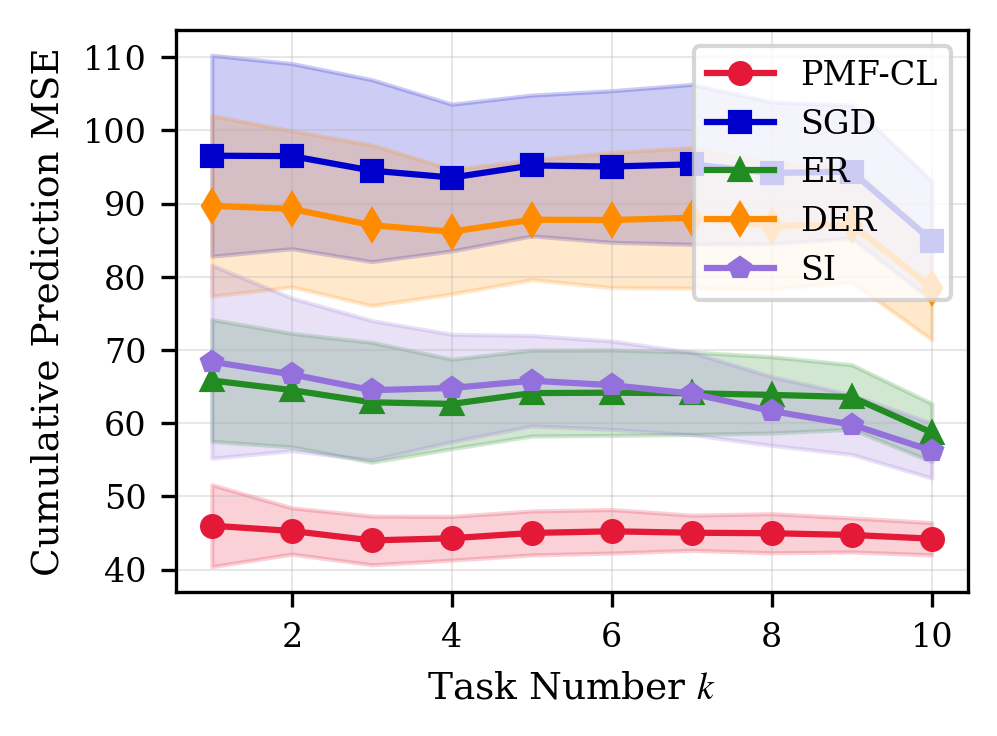}
        \caption{Cumulative Accuracy}
        \label{fig:cumulative_pred}
    \end{subfigure}
    \hfill
    \begin{subfigure}[b]{0.32\textwidth}
        \centering
        \includegraphics[width=\textwidth]{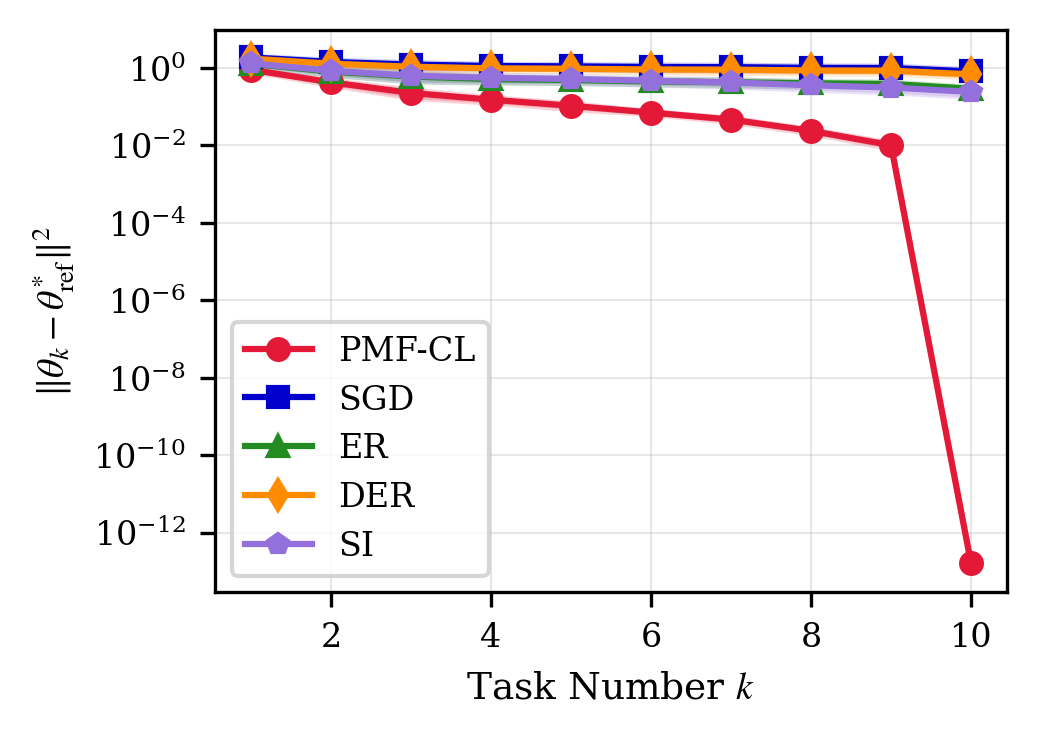}
        \caption{Parameter Learning Accuracy}
        \label{fig:param_error}
    \end{subfigure}
    \vspace{0.2cm}
    \caption{Performance metrics across 10 sequential tasks.
    (Red = PMF-CL, Blue = SGD, Green = ER, Orange = DER, Purple = SI).}
    \label{fig:prediction_accuracy}
\end{figure}

\textbf{Current Task Accuracy:} Figure \ref{fig:task_pred} reports the independent prediction error (MSE) evaluated sequentially after the conclusion of each task $k$.
PMF-CL consistently achieves the lowest MSE.
Crucially, the PMF-CL policy demonstrates no recency bias; it maintains uniform performance across the task sequence rather than skewing optimization strictly toward the most recently observed distribution. 

\textbf{Cumulative Forward Transfer:} Figure \ref{fig:cumulative_pred} measures the accumulated data prediction accuracy, calculated as the MSE over all datasets up to and including the current task $k$.
PMF-CL establishes a clear performance advantage by Task 3 and remains the dominant method through Task 10.
This trajectory highlights PMF-CL's strong resistance to catastrophic forgetting and its capacity for effective forward transfer.

\textbf{Convergence to the Pareto Optimum:} Figure \ref{fig:param_error} evaluates the MSE between the parameters learned after task $k$ and the global Pareto-optimal parameters (the theoretical optimum derived by jointly training on all data simultaneously).
Because PMF-CL explicitly optimizes for this combined dataset optimum, it successfully converges to the true Pareto-optimal solution upon the completion of all 10 tasks.
In contrast, all baseline methods fail to locate this global reference point, plateauing at suboptimal parameter configurations.

\subsection{Computational Complexity Analysis}
\label{subsec:complexity}

Continual learning algorithms must carefully balance predictive performance with resource constraints.
We analyze the memory and time complexities of the evaluated methods in Figure \ref{fig:complexity}.

\begin{figure}[ht]
    \centering
    \begin{subfigure}[b]{0.48\textwidth}
        \centering
        \includegraphics[width=\textwidth]{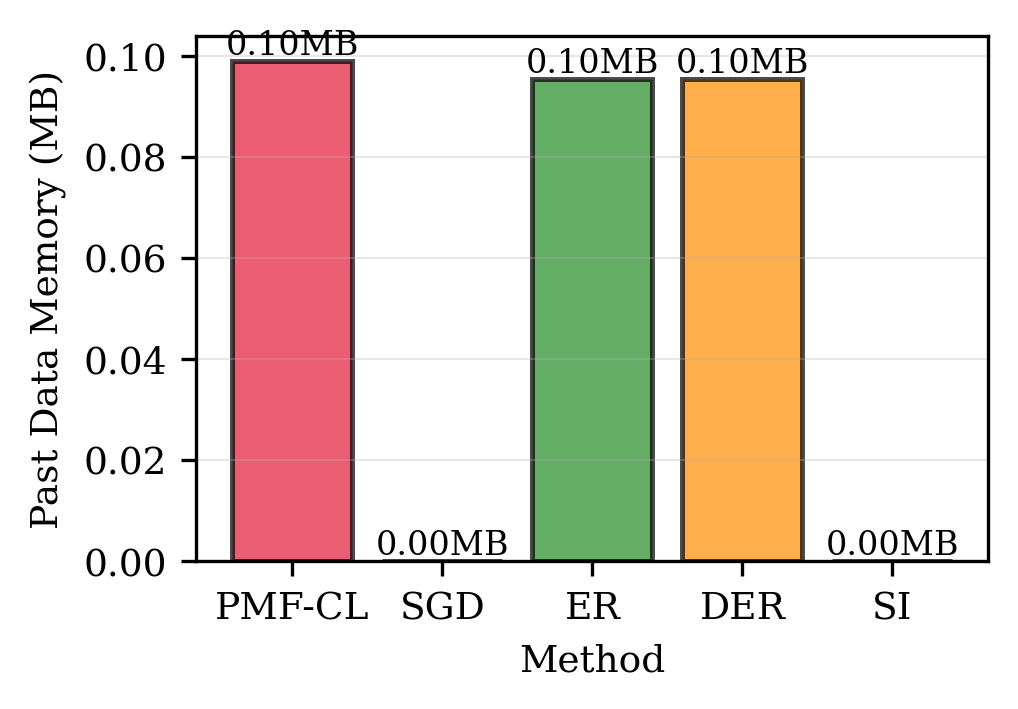}
        \caption{Rehearsal Memory Complexity}
        \label{fig:memory}
    \end{subfigure}
    \hfill
    \begin{subfigure}[b]{0.48\textwidth}
        \centering
        \includegraphics[width=\textwidth]{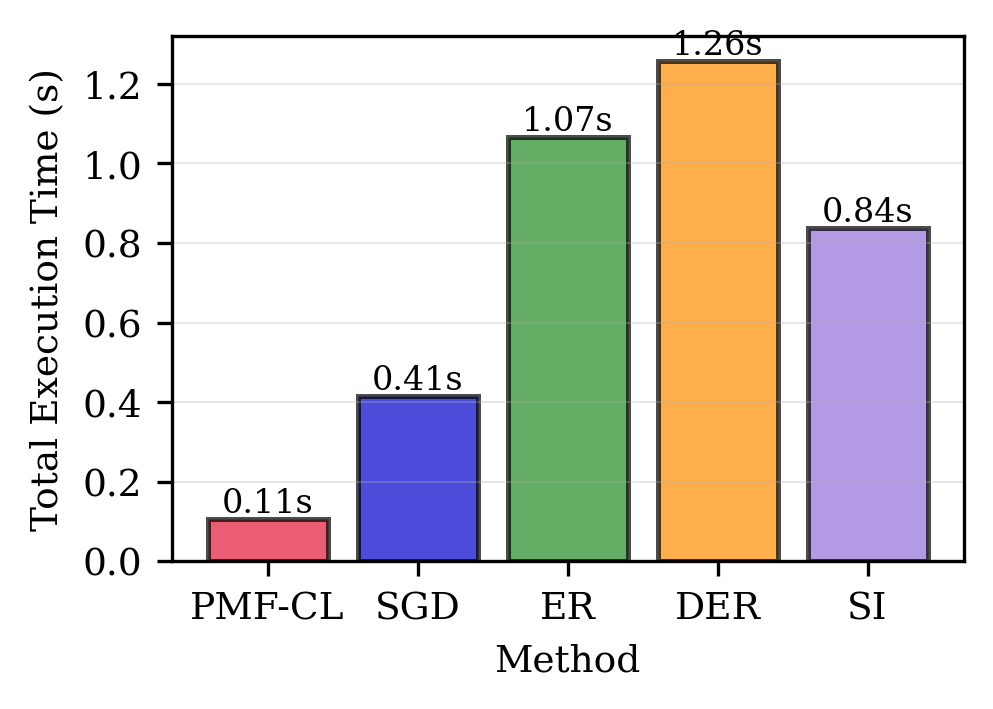}
        \caption{Time Complexity Comparison}
        \label{fig:time}
    \end{subfigure}
    \vspace{0.2cm}
    \caption{Computational trade-offs showing (a) total memory footprint and (b) execution time per method.}
    \label{fig:complexity}
\end{figure}

\textbf{Memory Footprint:} As shown in Figure \ref{fig:memory}, the memory complexity of PMF-CL is highly competitive, performing on par with standard rehearsal techniques (ER, DER) in this benchmark.
While the replay buffer size for ER/DER aligns with the data retained by our method in this specific setting, PMF-CL offers a distinct structural advantage: its memory footprint is deterministic and fixed entirely by the problem's dimensionality, whereas the memory requirements for rehearsal methods are heavily setting- and buffer-dependent.

\textbf{Execution Time:} Figure \ref{fig:time} demonstrates that PMF-CL executes highly efficiently on this benchmark (averaging 0.11s).
However, we note that the theoretical time complexity of PMF-CL is bounded by the exact Singular Value Decomposition (SVD) operations required for parameter updates.
Consequently, for significantly higher-dimensional problems, PMF-CL's execution time may scale poorly compared to SGD-based methods.
Future iterations of PMF-CL could integrate approximate SVD methods to drastically reduce this overhead.
Furthermore, it should be noted that the baseline execution times are heavily parameterized by the chosen number of training iterations per task, making relative speed comparisons partially dependent on hyperparameter configurations.

\subsection{Continual Logistic Regression on Sequential MNIST}
\label{subsec:seq_mnist_logistic}

We next evaluate the QUB extension of PMF-CL for multiclass logistic
regression on Sequential MNIST.  This experiment complements the exact
quadratic study above: it tests both the continual-learning behavior of the
method and whether its recursive updates recover the offline optimum of the
quadratic surrogate derived in \appref{logistic-reg}.

\textbf{Experimental setup:}
We use the class-incremental Sequential MNIST benchmark with five tasks and
two new classes per task, in the order $(0,1),(2,3),\ldots,(8,9)$.  The model
is a linear $784$-dimensional input layer with ten output logits.  Results are
averaged over ten runs (seeds 42--51), and shaded regions show $\pm1$ standard
deviation.  We compare PMF-CL with SGD, ER, DER, and SI, using one training
epoch per task for these standard baselines.  ER and DER use buffers of 500
examples.  PMF-CL uses the B\"ohning quadratic upper bound and an
$\ell_2$ coefficient of $10^{-4}$; its task-local minimizers and the offline
reference solutions are fitted to numerical convergence.  Here, the baseline
comparison is meant to be a controlled validation under a fixed
one-epoch budget, and we do not claim to have exhaustively fine-tuned all the
methods.

\subsubsection{Accuracy, Forgetting, and Surrogate Validation}

\begin{figure}[ht]
    \centering
    \begin{subfigure}[b]{0.32\textwidth}
        \centering
        \includegraphics[width=\textwidth]{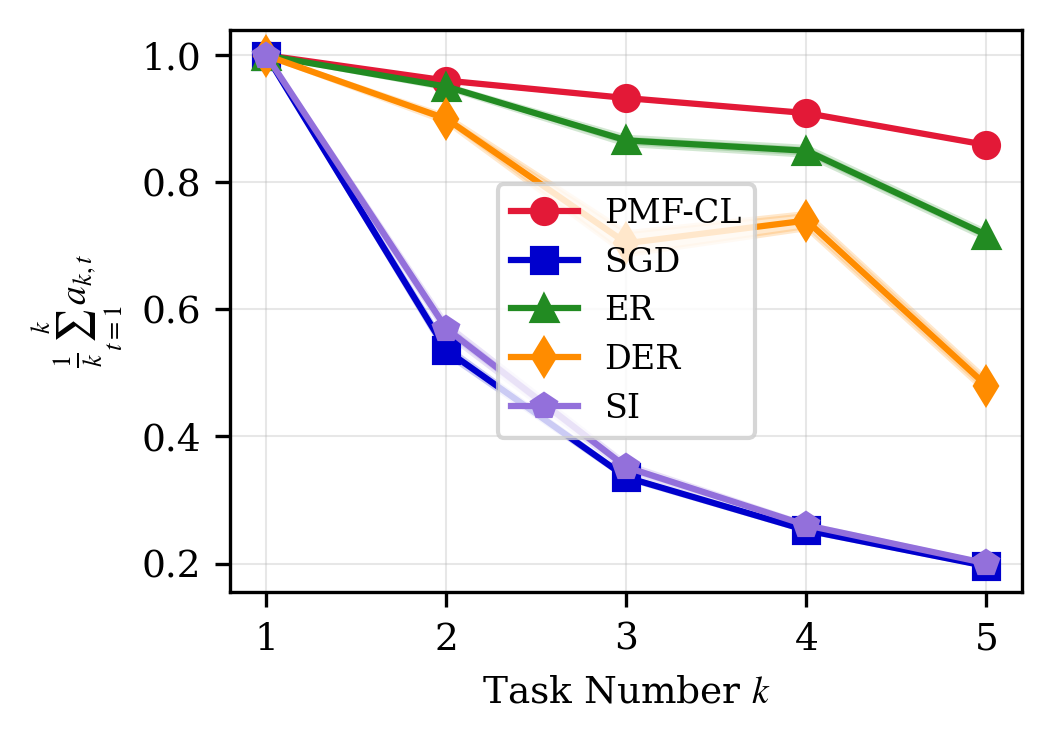}
        \caption{Cumulative accuracy}
        \label{fig:seq_mnist_cumulative_acc}
    \end{subfigure}
    \hfill
    \begin{subfigure}[b]{0.32\textwidth}
        \centering
        \includegraphics[width=\textwidth]{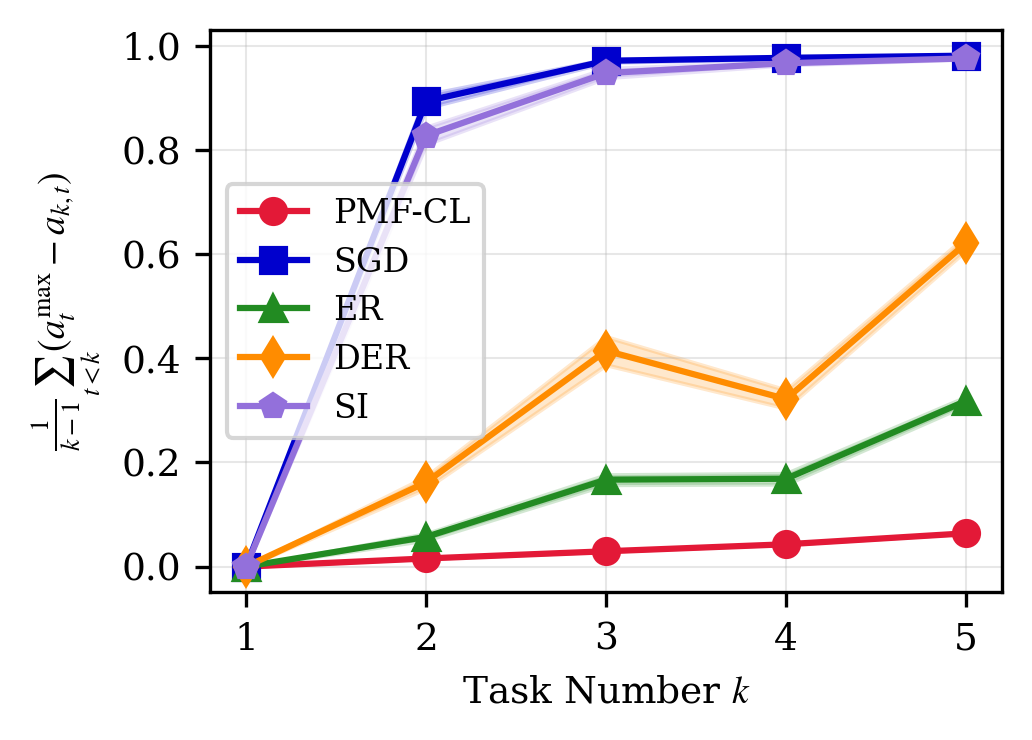}
        \caption{Average forgetting}
        \label{fig:seq_mnist_forgetting}
    \end{subfigure}
    \hfill
    \begin{subfigure}[b]{0.32\textwidth}
        \centering
        \includegraphics[width=\textwidth]{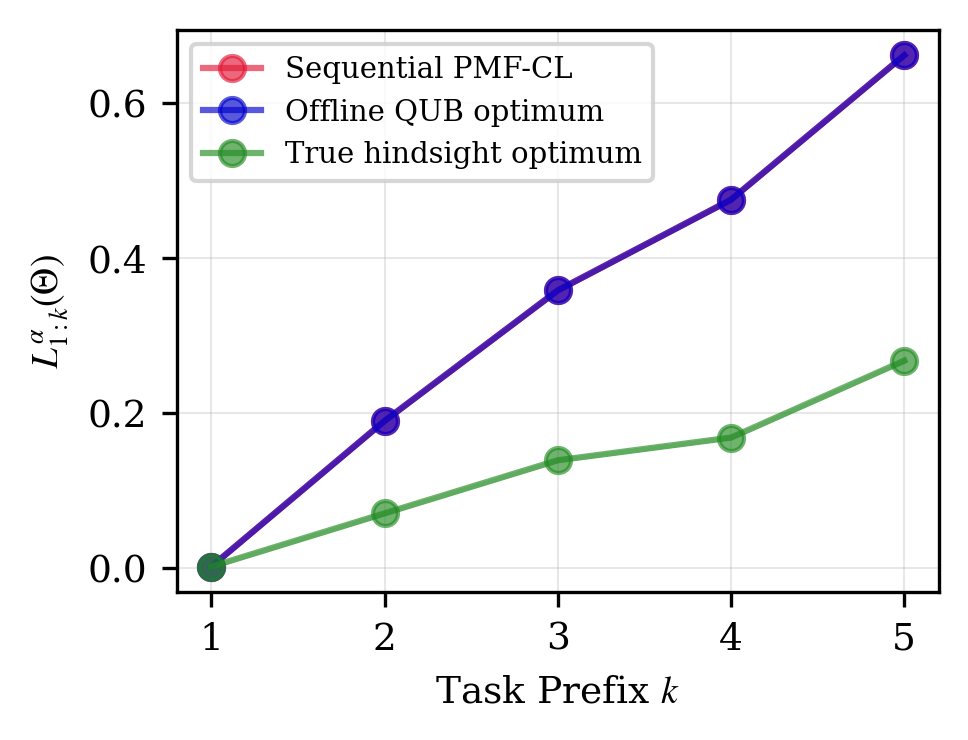}
        \caption{Scalarized true logistic loss}
        \label{fig:seq_mnist_true_loss}
    \end{subfigure}
    \caption{Sequential MNIST results over five two-class tasks.  Panels (a)
    and (b) compare PMF-CL with the continual-learning baselines.  Panel (c)
    evaluates the true sample-weighted logistic objective at the sequential
    PMF-CL iterate, the offline QUB optimum, and the true hindsight optimum for
    each task prefix.}
    \label{fig:seq_mnist_performance}
\end{figure}

\textbf{Cumulative accuracy and forgetting:}
Figure~\ref{fig:seq_mnist_cumulative_acc} reports the mean accuracy over all
tasks observed through prefix $k$.  PMF-CL retains $85.85\%$ cumulative
accuracy after the fifth task, compared with $71.69\%$ for ER, $48.00\%$ for
DER, $20.04\%$ for SI, and $19.61\%$ for naive SGD.  The corresponding
forgetting curves in Figure~\ref{fig:seq_mnist_forgetting} show the same
stability advantage: final average forgetting is $0.0639$ for PMF-CL, versus
$0.3180$ for ER, $0.6209$ for DER, $0.9754$ for SI, and $0.9806$ for SGD.
This gain in stability comes with reduced plasticity on the newest task:
PMF-CL's final current-task accuracy is $78.52\%$, while the four baselines
obtain between $95.34\%$ and $96.37\%$.

\textbf{Recovery of the offline surrogate solution:}
In Figure~\ref{fig:seq_mnist_true_loss}, the sequential PMF-CL and offline-QUB
curves visually coincide at every prefix.  Direct comparison of their
parameters confirms that the recursion error remains below the prescribed
$10^{-8}$ tolerance, with final mean-squared parameter error
$2.47\times10^{-11}\pm4.50\times10^{-11}$.  Hence, the continual sufficient-
information update faithfully recovers the solution that would be obtained by
optimizing the accumulated quadratic surrogate offline.

The green curve separates this implementation check from the approximation
introduced by the QUB.  At prefix five, the sequential/offline-QUB solution
has true sample-weighted logistic loss $0.6612$, whereas the true hindsight
logistic optimum attains $0.2675$; their parameter MSE is $0.160$.  All
measured QUB-minus-true-loss gaps were nonnegative.  Therefore, the remaining
discrepancy is attributable to looseness of the quadratic surrogate, rather
than error in the recursive PMF-CL update.

\subsubsection{Computational Trade-offs}

\begin{figure}[H]
    \centering
    \begin{subfigure}[b]{0.48\textwidth}
        \centering
        \includegraphics[width=\textwidth]{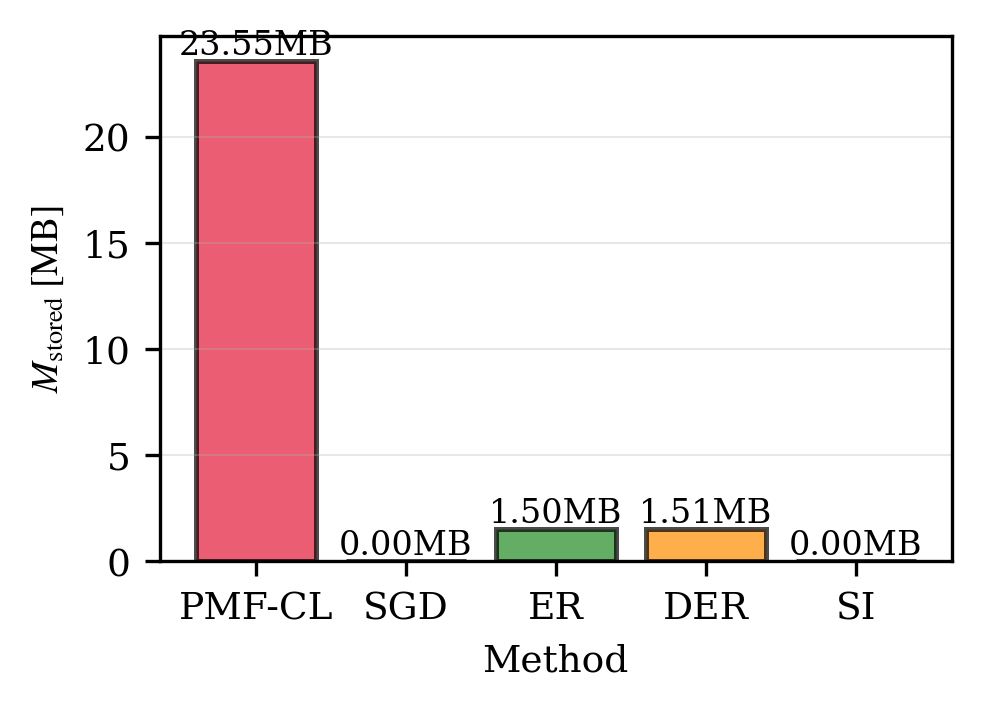}
        \caption{Stored memory}
        \label{fig:seq_mnist_memory}
    \end{subfigure}
    \hfill
    \begin{subfigure}[b]{0.48\textwidth}
        \centering
        \includegraphics[width=\textwidth]{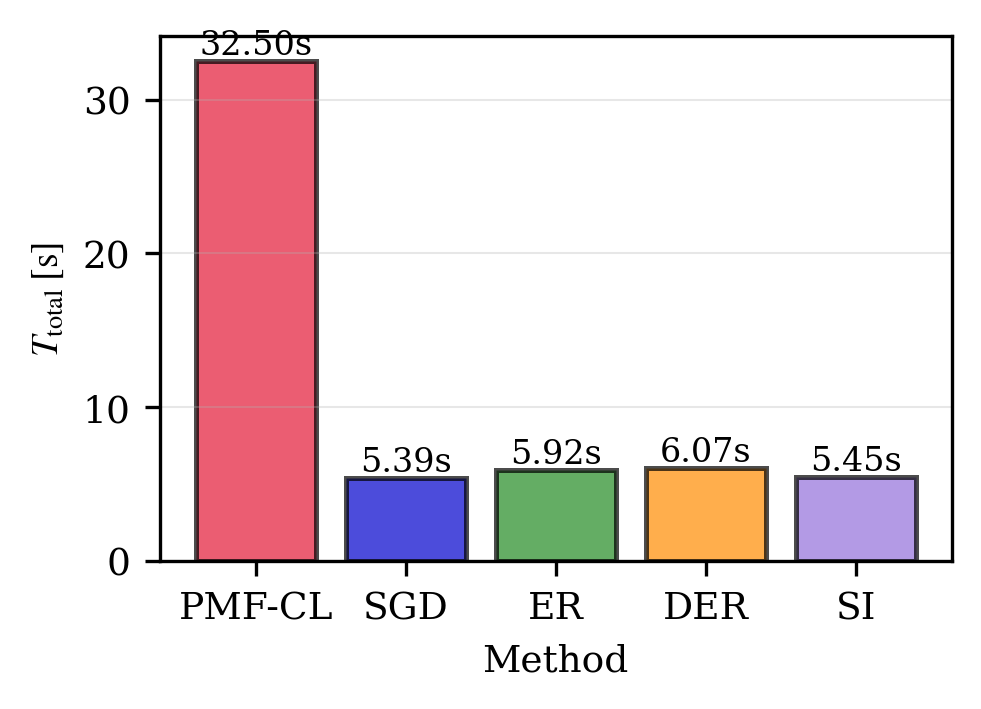}
        \caption{Total runtime}
        \label{fig:seq_mnist_time}
    \end{subfigure}
    \caption{Resource usage for the Sequential MNIST experiment.  PMF-CL
    stores fixed-dimensional sufficient information, while ER and DER store
    their configured replay buffers; SGD and SI have no replay memory.}
    \label{fig:seq_mnist_complexity}
\end{figure}

Figure~\ref{fig:seq_mnist_complexity} quantifies the cost of the exact
surrogate updates in this implementation.  PMF-CL stores $23.55$ MB, compared
with $1.50$ MB for ER and $1.51$ MB for DER, and takes $32.50$ s in total,
versus approximately $5.39$--$6.07$ s for the baselines.  The larger cost is
consistent with storing and manipulating the quadratic sufficient statistics
for a $784\times10$ linear classifier and with fitting the task-local
surrogates to convergence.  Importantly, this state is determined by model
dimension and does not grow with the number of tasks, whereas replay memory is
controlled by a user-selected buffer size.  The wall-clock comparison also
depends on the unequal optimization budgets described above and is included
to expose this computational trade-off, not to claim a speed advantage.




\end{document}